\documentclass[letterpaper]{article} %
\usepackage{aaai2026}  %
\usepackage{times}  %
\usepackage{helvet}  %
\usepackage{courier}  %
\usepackage[hyphens]{url}  %
\usepackage{graphicx} %
\usepackage{natbib}  %
\usepackage{caption} %
\usepackage{algorithm}
\usepackage{algorithmic}

\usepackage{newfloat}
\usepackage{listings}
\DeclareCaptionStyle{ruled}{labelfont=normalfont,labelsep=colon,strut=off} %
\floatstyle{ruled}
\newfloat{listing}{tb}{lst}{}
\floatname{listing}{Listing}
\title{Statistical Learning Theory for Distributional Classification}
\author{
    Christian Fiedler
}
\affiliations{
    Department of Mathematics, School of Computation, Information and Technology, Technical University of Munich (TUM) \\
    Munich Center for Machine Learning (MCML)\\
    Institute for Data Science in Mechanical Engineering (DSME)\thanks{Majority of the work conducted while at DSME.}, RWTH Aachen University \\
    christian.fiedler@tum.de
}

\usepackage{amsmath}
\usepackage{amssymb}
\usepackage{amsthm}

\usepackage{cleveref}

\newcommand{\R}{\mathbb{R}}
\newcommand{\Rnn}{\R_{\geq 0}}
\newcommand{\Rp}{\R_{> 0}}
\newcommand{\Np}{\mathbb{N}_+}

\newcommand{\Pb}{\mathbb{P}}
\newcommand{\E}{\mathbb{E}}
\newcommand{\setsys}[1]{\mathcal{#1}}
\newcommand{\distr}{\sim}
\newcommand{\iiddistr}{\overset{\text{i.i.d.}}{\sim}}

\newcommand{\Borel}{\mathcal{B}} %

\newcommand{\MeasurableFunctions}{{\mathcal{L}^0}}

\newcommand{\fm}{\Phi} %
\newcommand{\fs}{\mathcal{H}} %
\newcommand{\hs}[1]{\mathcal{#1}} %
\newcommand{\KME}[1]{\Pi_{#1}} %

\newcommand{\cK}{\mathcal{K}} %
\newcommand{\cL}{\mathcal{L}} %

\newcommand{\samplingSet}{\mathcal{S}} %
\newcommand{\inputSet}{\mathcal{X}} %
\newcommand{\outputSet}{\mathcal{Y}} %

\newcommand{\dataSet}{\mathcal{D}}
\newcommand{\dataSetBar}{\bar{\dataSet}}

\newcommand{\svm}[2]{f_{{#2},{#1}}} %
\newcommand{\risk}{\mathcal{R}} %
\newcommand{\Risk}[2]{\risk_{{#1},{#2}}} %
\newcommand{\RiskBayes}[2]{\risk_{{#1},{#2}}^\ast} %
\newcommand{\RiskReg}[3]{\risk_{{#1},{#2},{#3}}} %
\newcommand{\RiskRegOpt}[4]{\risk_{{#1},{#2},{#3}}^{#4\ast}} %

\newcommand{\LearningAlg}[1]{\mathcal{L}_{#1}}

\newcommand{\distributions}{\mathcal{M}_1} %

\newcommand{\LocLip}[2]{|#1|_{1,{#2}}} %

\newcommand{\hilbertianHS}{\hs{H}} %
\newcommand{\hilbertianEmbed}{\Pi} %

\newcommand{\kKME}{\kappa} %
\newcommand{\kTwo}{k} %
\newcommand{\KMEk}{\KME{\kKME}} %
\newcommand{\KMEhat}{\hat{\Pi}_\kKME} %

\newcommand{\scp}{\langle\cdot,\cdot\rangle}
\newcommand{\Normal}{\mathcal{N}}
\newcommand{\tr}{\mathrm{tr}} %
\newcommand{\C}{\mathbb{C}}

\newcommand{\pipe}{\lhd}

\newcommand{\Rademacher}{\mathrm{Rad}}

\newcommand{\regularizer}{\Omega}
\newcommand{\ApproxErrorFunc}[3]{A^{#3}_{#1,#2}}

\newcommand{\Aef}{\ApproxErrorFunc{\ell}{P_\hilbertianEmbed}{H_k}}

\newcommand{\RiskOpt}[3]{\Risk{#1}{#2}^{#3\ast}}

\newcommand{\hilbertianEmbedBound}{B_\hilbertianEmbed} %

\newcommand{\cKi}{\cK_\infty} %
\newcommand{\cKL}{\mathcal{KL}} %

\newcommand{\K}{\mathbb{K}} %

\newcommand{\clipped}[1]{\bar{#1}}

\newcommand{\defm}[1]{\emph{#1}}
\newcommand{\clipValue}{M}

\newcommand{\stp}[2]{{\overset{(#2)}{#1}}}
\newcommand{\stpx}[1]{{(#1)}}

\newcommand{\ellClass}{\ell_{\mathrm{c}}} %
\newcommand{\ellHinge}{\ell_{\mathrm{h}}} %

\newcommand{\gaussianKernel}[1]{k_{#1}}

\newcommand{\sgn}{\mathrm{sgn}}

\newtheorem{theorem}{Theorem}
\newtheorem{proposition}[theorem]{Proposition}
\newtheorem{lemma}[theorem]{Lemma}

\newtheorem{assumption}[theorem]{Assumption}
\newtheorem{remark}[theorem]{Remark}
\newtheorem{example}[theorem]{Example}

\newif\ifTwoColumns
\TwoColumnstrue

\newif\ifNotMuchSpace
\NotMuchSpacetrue

\newif\ifCameraReady
\CameraReadytrue

\begin{document}

\maketitle
\begin{abstract}
In supervised learning with distributional inputs in the two-stage sampling setup, relevant to applications like learning-based medical screening or causal learning, the inputs (which are probability distributions) are not accessible in the learning phase, but only samples thereof. This problem is particularly amenable to kernel-based learning methods, where the distributions or samples are first embedded into a Hilbert space, often using kernel mean embeddings (KMEs), and then a standard kernel method like Support Vector Machines (SVMs) is applied, using a kernel defined on the embedding Hilbert space. In this work, we contribute to the theoretical analysis of this latter approach, with a particular focus on classification with distributional inputs using SVMs. We establish a new oracle inequality and derive consistency and learning rate results. Furthermore, for SVMs using the hinge loss and Gaussian kernels, we formulate a novel variant of an established noise assumption from the binary classification literature, under which we can establish learning rates. Finally, some of our technical tools like a new feature space for Gaussian kernels on Hilbert spaces are of independent interest.
\end{abstract}

\section{Introduction} \label{sec:distrslt:intro}
In supervised learning, distributions can appear as inputs, and this scenario has been considered in many works, cf. \cite{muandet2012learning} and the references therein.
However, in some applications, the distributions acting as inputs are not directly accessible, but only samples thereof.
For example, in the context of artificial intelligence (AI) assisted medical diagnosis, one might want to train on past patient data a binary classifier acting on biomarkers in order to detect an illness or anomaly.
In practice, these biomarkers might not be fully observable, but only samples from them are available, e.g., through repeated measurements.
In turn, we can model the biomarkers as distributions, which are only accessible through samples \cite{szabo2015two}.
A concrete instance of this situation is training a classifier to detect atrial fibrillation from electrocardiogram measurements \cite{massiani2025robust}.
Another example can be found in statistical learning approaches to causal learning, where a classifier for the direction of causality between two random variables is desired, and such a causality classifier can be trained on samples from distributions with known causality structure \cite{lopezpaz2015towards}.
The common feature of these examples is the \emph{two-stage sampling} setup.
First, a distribution (the input) and some output is sampled from a population, and then samples from the distribution acting as an input are drawn, and only these samples (and the output) are available during the learning phase.
This setup has received particular attention in connection with kernel methods.
Commonly, the distributions and the samples thereof are first embedded into a Hilbert space,
and then a standard kernel method (now with inputs from a Hilbert space) is used on the transformed data set.
The learned hypothesis can then be used on distributional inputs by composing it with the embedding \cite{szabo2015two}.
This strategy has received a lot of attention, especially in the context of regression problems, where a Hilbertian embedding is combined with kernel ridge regression (KRR), and a substantial body of theory is available, including learning rates \cite{szabo2015two,szabo2016learning}.
However, for some applications, other types of learning problems might be more appropriate, and a theoretical analysis should reflect this.
For instance, the two examples described above are most naturally framed as classification problems, and \cite{massiani2025robust} actually used a classification SVM instead of KRR, with excellent empirical results.

Providing an analysis tailored to the classification setting is particularly important with regards to the assumptions used therein.
As is well-known, in order to establish learning rates, one has to make distributional assumptions due to the \emph{No Free Lunch-Theorem} \cite[Chapter~6]{SC08}.
For regression problems, typically a certain smoothness of the regression function is assumed, which is natural and reasonable in this setting. 
However, a smoothness assumption might not be appropriate for classification problems, since it does not necessarily capture the intrinsic difficulty (or simplicity) of a classification problem.
Instead, margin and noise exponent assumptions are more appropriate, cf. \cite{steinwart2007fast} and \cite[Chapter~8]{SC08} for an overview and discussion.
In turn, this calls for an analysis that can take these considerations into account.
To the best of our knowledge, consistency and learning rate results for kernel-based distributional classification in the two-stage sampling setup under assumptions natural for the classification setting are still missing.
In this work, we close this gap by providing a thorough theoretical investigation of SVMs with distributional inputs in the two-stage sampling setup.

\paragraph{Related Work}
After its introduction in \cite{poczos2013distribution}, learning in the two-stage sampling setup has been primarily investigated in the context of kernel methods, starting with \cite{szabo2015two}, which uses kernel mean embeddings (KMEs) for the Hilbertian embeddings of probability distributions.
Note that \cite{muandet2012learning} is an earlier work that uses this embedding approach, but not in the two-stage sampling setup.
While primarily KMEs have been used for the embeddings, other Hilbertian embeddings have been considered, like using sliced Wasserstein kernels \cite{meunier2022distribution}, and a variety of such embeddings are available, with \cite{bonnier2023kernelized} as a recent example, and even learned embeddings \cite{kachaiev2024learning}.
The main focus has been on regression, with KRR as the kernel method used after the embedding, and this setting is by now well-understood \cite{szabo2016learning,fang2020optimal}, different algorithmic approaches \cite{mucke2021stochastic}, and robust variants \cite{yu2021robust}.
Despite its practical importance, distributional classification in the two-stage setup has received much less attention.
This setting is considered in \cite{lopezpaz2015towards}, though this work focuses on empirical risk minimization.
In \cite{fiedler2024statistical}, first steps towards a more general learning-theoretic analysis are taken, including oracle inequalities for SVMs in the two-stage setups.
In particular, it was recognized that most existing works rely on the integral operator technique \cite{caponnetto2007optimal}, which is limited to KRR (and related spectral regularization techniques for regression) and hence cannot be used to treat the classification setting, using for example SVMs with the hinge loss.
However, no consistency results or learning rates are provided, and the oracle inequalities have technical limitations, requiring significant regularity of the loss function, which excludes the hinge loss, or requiring a suitable discretization of the hypothesis space, which can be problematic due to the Hilbertian embedding.
Finally, we would like to stress that we focus on kernel-based approaches for learning with distributional inputs in the two-stage sampling setup, and other approaches have been considered in this context like deep learning, cf. \cite{liu2025generalization} for an example and further pointers to the literature, and we refer to \cite[Section~C]{kachaiev2024learning} for an overview.

\paragraph{Contributions}
As a starting point, we prove a very general oracle inequality (Theorem \ref{thm:distrslt:oracleInequSVM}) for SVMs in the two-stage sampling setup.
In particular, in contrast to the results in \cite{fiedler2024statistical} it can now handle the hinge loss without any discretization of the hypothesis space.
We then prove (universal) consistency for rather general loss functions, which includes classification with the hinge loss, both for generic Hilbertian embeddings (Proposition \ref{prop:distrslt:consistencySVM}) and kernel mean embeddings (Proposition \ref{prop:distrslt:consistencySVMwithKMEs}).
Furthermore, under a standard assumption we then establish learning rates for rather general loss functions (Theorem \ref{thm:distrslt:learningRateSVMwithKMEs}), again with classification as a special case.
To the best of our knowledge, these are the first such results for distributional classification with SVMs in the two-stage setup.
Finally, we investigate learning rates for distributional classification in the two-stage sampling setting with SVMs using the hinge loss and Gaussian kernels, under assumptions that are natural for classification problems.
For this, we introduce a variant of a well-known geometric noise exponent assumption from the theory of binary classification (Assumption \ref{assumption:distrslt:geometricNoiseAssumptionGaussian}),
which in turn allows us to establish learning rates (Theorem \ref{thm:distrslt:learningRateClassificationGaussianKernel}), %
without any explicit smoothness assumption as used in regression.
To the best of our knowledge, this is again the first such result.
Furthermore, for the proof Theorem \ref{thm:distrslt:learningRateClassificationGaussianKernel} we develop a new feature space for Gaussian kernels on Hilbert spaces (Theorem \ref{thm:distrslt:featureSpaceGaussianKernelOnHS}), which is of independent interest.
Due to space constraints, most of the proofs of our results have been placed in the supplementary material.

\section{Preliminaries} \label{sec:distrslt:prelims}
We start by recalling some preliminaries, including the classic setup of statistical learning theory, and the type of learning methods we consider.
\paragraph{General background}
We follow \cite{fiedler2024statistical} and use \defm{comparison functions} as common in control theory.
Recall that class $\cK$ functions are defined as
\ifTwoColumns
$\cK=\{f:\Rnn\rightarrow\Rnn \mid f \text{ continuous, strictly increasing}, f(0)=0 \}$,
\else
\begin{equation*}
    \cK=\{f:\Rnn\rightarrow\Rnn \mid f \text{ continuous, strictly increasing}, f(0)=0 \},
\end{equation*}
\fi
and relations and operations on $\cK$ are defined pointwise.
For the reader's convenience, we have collected additional background on comparison functions 
\ifCameraReady
in the supplementary material.
\fi

We denote the real part of a complex number $z\in\C$ by $\Re z$.
For a measure space $(\Omega,\setsys{A},\mu)$ and $\K=\{\R,\C\}$, let $L^2(\Omega,\mu;\K)$ be the usual Lebesgue space of ($\mu$-equivalence classes of) $\K$-valued square-integrable functions, and let $L^2_\R(\Omega,\mu;\C)$ be the real Hilbert space arising by restricting scalar multiplication in $L^2(\Omega,\mu;\C)$ to the reals.
We denote by $L_1^+(\hs{H})$ the set of continuous, linear, self-adjoint, trace-class, positive operators on a Hilbert space $\hs{H}$.

Finally, for $Q\in L_1^+(\hs{H})$ we denote by $\Normal(0,Q)$ the Gaussian measure on $\hs{H}$ with covariance operator $Q$,
and by $\hs{H} \ni h \mapsto W_h \in L^2(\hs{H},\Normal(0,Q);\R)$ the associated white noise mapping, cf. \cite{daprato2002second} and 
\ifCameraReady
the supplementary material
\fi
for more details.
\paragraph{Statistical Learning Theory}
We follow the standard setup as formalized in \cite[Chapters~2,~6]{SC08}.
The \defm{input space} is a measurable space $\inputSet$,
the \defm{output space} $\outputSet\subseteq\R$ closed with the corresponding Borel $\sigma$-algebra,
and we consider only \defm{supervised loss functions}, i.e., measurable functions $\ell:\outputSet\times\R\rightarrow\Rnn$,
which we call continuous, differentiable etc.\ if $\ell(y,\cdot)$ has this property for all $y\in\outputSet$.
Furthermore, we define
\ifNotMuchSpace
$\LocLip{\ell}{T}=\sup\{|\ell(y,t_1)-\ell(y,t_2)|/|t_1-t_2| \mid t_1,t_2\in[-T,T], t_1\not=t_2, y\in\outputSet\}$
\else
\begin{equation*}
    \LocLip{\ell}{T}=\sup_{\substack{t_1,t_2\in[-T,T]\\t_1\not=t_2,y\in\outputSet}} \frac{|\ell(y,t_1)-\ell(y,t_2)|}{|t_1-t_2|}
\end{equation*}
\fi
and say that $\ell$ is \defm{locally Lipschitz-continuous}\footnote{Note that this definition entails uniformity in the first argument of $\ell$.} if $\LocLip{\ell}{T}<\infty$ for all $T\in\Rp$.
We call $\ell$ \defm{globally $L_\ell$-Lipschitz continuous} for $L_\ell\in\Rnn$, if $|\ell(y,t)-\ell(y,t')|\leq L_\ell|t-t'|$ for all $y\in\outputSet$, $t,t'\in\R$.
We say that $\ell$ can be \defm{clipped at $\clipValue\in\Rp$} if
\ifNotMuchSpace
$\ell(y,\clipped{t}) \leq \ell(x,y,t)$ for all $y\in\outputSet, t\in\R$,
\else
\begin{equation*}
    \ell(y,\clipped{t}) \leq \ell(y,t) \quad \forall y\in\outputSet, t\in\R,
\end{equation*}
\fi
where
\begin{equation*}
    \clipped{t} = \begin{cases}
        \clipValue & \text{if } t > M \\
        t & \text{if } -M \leq t \leq M \\
        -\clipValue & \text{if } t < -M
    \end{cases}
\end{equation*}
is the \defm{clipped value} of $t$.
Furthermore, for a map $f: \inputSet \rightarrow\R$, we define $\ell \pipe f: \inputSet \times\outputSet \rightarrow \Rnn$ by $(\ell \pipe f)(x,y) = \ell(y,f(x))$.
For a \defm{data-generating distribution} $P$ on $\inputSet\times\outputSet$, 
we define the \defm{risk} of a \defm{hypothesis} $f:\inputSet\rightarrow\R$ (measurable function) as
\ifNotMuchSpace
$\Risk{\ell}{P}(f)=\int_{\inputSet\times\outputSet} \ell(y,f(x))\mathrm{d}P(x,y)$,
\else
\begin{equation*}
    \Risk{\ell}{P}(f)=\int_{\inputSet\times\outputSet} \ell(y,f(x))\mathrm{d}P(x,y),
\end{equation*}
\fi
the \defm{Bayes risk} as
\ifNotMuchSpace
$ \RiskBayes{\ell}{P} = \inf\{ \Risk{\ell}{P}(f) \mid f:\inputSet\rightarrow\R, \: \text{measurable}\}$,
\else
\begin{equation*}
     \RiskBayes{\ell}{P} = \inf_{\substack{f:\inputSet\rightarrow\R\\\text{measurable}}} \Risk{\ell}{P}(f),
\end{equation*}
\fi
and for a hypothesis class $H$ also
\ifNotMuchSpace
$ \RiskOpt{\ell}{P}{H} = \inf\{ \Risk{\ell}{P}(f) \mid f \in H \}$.
\else
\begin{equation*}
     \RiskOpt{\ell}{P}{H} = \inf_{f \in H_k} \Risk{\ell}{P}(f).
\end{equation*}
\fi
For a \defm{data set} $\dataSet=((x_1,y_1),\ldots,(x_N,y_N))\in(\inputSet\times\outputSet)^N$,
we define the \defm{empirical risk} as
\ifNotMuchSpace
$ \Risk{\ell}{\dataSet} = \frac1N \sum_{n=1}^N \ell(y_n,f(x_n))$.
\else
\begin{equation*}
     \Risk{\ell}{\dataSet} = \frac1N \sum_{n=1}^N \ell(y_n,f(x_n)).
\end{equation*}
\fi
If $H$ is a normed vector space, we define the \defm{regularized risk} with \defm{regularization parameter} $\lambda\in\Rp$ as
$\RiskReg{\ell}{P}{\lambda}(f) = \Risk{\ell}{P}(f) + \lambda\|f\|_H^2$,
and the \defm{regularized empirical risk} as 
$\RiskReg{\ell}{\dataSet}{\lambda}(f) = \Risk{\ell}{\dataSet}(f) + \lambda\|f\|_H^2$.
Finally, we turn to notions of learnability.
A \defm{learning method} is a measurable\footnote{For precise definitions, we refer to \cite[Chapter~6]{SC08}.} map between data sets and a hypothesis space $H$, formally $\bigcup_{N\in\Np} (\inputSet\times\outputSet)^N \ni \dataSet \mapsto f_\dataSet \in H$.
We call such a learning method \defm{$\ell$-risk consistent} or just \defm{consistent} if $\Risk{\ell}{P}(f_{\dataSet_N}) \rightarrow \RiskBayes{\ell}{P}$ in probability for $N\rightarrow \infty$ and $\dataSet_N\distr P^{\otimes N}$, and \defm{universally $\ell$-risk consistent} or just \defm{universally consistent} if this holds for all data-generating distributions $P$ on $\inputSet\times\outputSet$.
Given a set $\mathcal{P}$ of distributions on $\inputSet\times\outputSet$, a \defm{learning rate} for the learning method is a sequence $(\epsilon_N)_N\subseteq\Rnn$ together with a constant $C_\mathcal{P}$ and $(c_\delta)_{\delta\in(0,1]}$ such that for all $P\in\mathcal{P}$, $N\in\Np$, and $\delta\in(0,1]$ we have
\begin{equation}
    \Pb_{\dataSet\distr P^{\otimes N}}[\Risk{\ell}{P}(f_\dataSet) \leq \RiskBayes{\ell}{P} + C_\mathcal{P}c_\delta\epsilon_N] \geq 1-\delta.
\end{equation}
\begin{example} \label{example:distrslt:classificationLoss}
We are primarily interested in binary classification, which can be formlized with $\outputSet=\{-1,1\}$ (encoding the two classes)
and the 0-1-loss $\ellClass:\outputSet\times\R\rightarrow\Rnn$, defined by
\begin{equation*}
    \ellClass(y,t) = \begin{cases}
        0 & \text{if } \sgn(t) = y \\
        1 & \text{otherwise}
    \end{cases}
\end{equation*}
\ifNotMuchSpace
where $\sgn(t)=1$ if $t\geq 0$, and $-1$ otherwise.
\else
\begin{equation*}
    \sgn(t) = \begin{cases}
        1 & \text{if } t \geq 0 \\
        -1 & \text{otherwise}
    \end{cases}
\end{equation*}
\fi
Note that $\ellClass$ is discontinuous and nonconvex.
\end{example}
\paragraph{Kernels and SVMs}
We now collect some well-known definitions and facts related to kernels, reproducing kernel Hilbert spaces, and support vector machines, based on the exposition in \cite{SC08}, and we refer to this reference for more details.
Recall that a function $k:X\times X\rightarrow\R$ is called a \defm{kernel} on an arbitrary nonempty set $X$, if there exists a Hilbert space $\fs$ (called \defm{feature space}) and a map $\fm: X \rightarrow\fs$ (called \defm{feature map})
such that
\ifNotMuchSpace
$k(x,x')=\langle \fm(x'), \fm(x) \rangle_\fs$ for all $x,x'\in X$.
\else
\begin{equation*}
    k(x,x')=\langle \fm(x'), \fm(x) \rangle_\fs \quad \forall x,x'\in X.
\end{equation*}
\fi
Furthermore, $k$ is the \defm{reproducing kernel} of a Hilbert space $H$ of functions on $X$ if $k(\cdot,x)\in H$ for all $x\in X$, and $f(x)=\langle f, k(\cdot,x)\rangle_H$ for all $f\in H$, $x\in X$, and $H$ is called a \defm{reproducing kernel Hilbert space} (RKHS) if it has a reproducing kernel (which is then unique).
Recall also that $k$ is a kernel if and only if it is the reproducing kernel of an RKHS, and the latter is then unique and denoted by $(H_k,\|\cdot\|_k)$.
Observe that then $H_k$ is a feature space for $k$ with feature map $\fm_k(x)=k(\cdot,x)$ (called \defm{canonical feature map}).
If $(\fs,\fm)$ is an arbitrary feature space-feature map pair for a kernel $k$, then $\fs \ni h \mapsto \langle h, \fm(\cdot) \rangle_\fs \in H_k$ is a canonical surjection on the RKHS of $k$.
We also use the notation $\|k\|_\infty=\sup_{x\in X} \sqrt{k(x,x)}$ and remark that $k$ is bounded if and only if $\|k\|_\infty<\infty$.

In this work, we consider \defm{regularized empirical risk minimization} (RERM) over an RKHS $H_k$ for a kernel $k$ on $\inputSet$, which is called a \defm{support vector machine} (SVM) in this context.
For a data set $\dataSet\in(\inputSet\times\outputSet)^N$, this corresponds to
\begin{equation*}
    \inf_{f \in H_k} \RiskReg{\ell}{\dataSet}{\lambda}(f),
\end{equation*}
and for $\ell$ convex there exists a unique solution denoted by $\svm{\lambda}{\dataSet}^{H_k}$, or $\svm{\lambda}{\dataSet}$ if $H_k$ is clear from the context.
For analysis purposes, we also define the \defm{approximation error function}
\ifNotMuchSpace
$\ApproxErrorFunc{\ell}{P}{H_k}(\lambda)=\RiskRegOpt{\ell}{P}{\lambda}{H_k} - \RiskOpt{\ell}{P}{H_k}$.
\else
\begin{equation*}
    \ApproxErrorFunc{\ell}{P}{H_k}(\lambda)=\RiskRegOpt{\ell}{P}{\lambda}{H_k} - \RiskOpt{\ell}{P}{H_k}
    = \left(\inf_{f\in H_k} \Risk{\ell}{P}(f) + \lambda\|f\|_k^2 \right) - \inf_{f \in H_k} \Risk{\ell}{P}(f).
\end{equation*}
\fi
\begin{example} \label{example:distrslt:hingeLoss}
We are primarily interested in classification, which can be described by the 0-1-loss $\ellClass$, cf. Example \ref{example:distrslt:classificationLoss}.
However, since $\ellClass$ is discontinuous and nonconvex, it is in practice not suitable for RERM, and instead \defm{surrogate losses} are used.
The idea is that these losses are well-behaved (in particular, convex), yet still describe a classification task, which can be formalized by \defm{classification calibration}, cf. \cite[Chapters~2,3]{SC08} and \cite[Chapter~4]{bach2024learning}, as well as 
\ifCameraReady
the supplementary material
\fi
for more background on this.
The most important example in our context is the \defm{hinge loss} $\ellHinge: \outputSet\times\R\rightarrow\Rnn$,
defined by $\ellHinge(y,t)=\max\{0, 1-yt\}$, which is convex and globally Lipschitz continuous, and can be clipped at 1.
\end{example}

\section{Setup} \label{sec:distrslt:setup}
In this section, we formalize the precise setting we work with for the remainder of this manuscript.
We first describe the general two-stage learning setup,
then we introduce abstract Hilbertian embeddings, and finally we outline kernel mean embeddings as a concrete examle of suitable Hilbertian embeddings.
\paragraph{Two-stage sampling}
We now formalize the two-stage sampling setup for distributional inputs as introduced in \cite{poczos2013distribution,szabo2015two}, following the formalization from \cite{fiedler2024statistical}. %
The underlying sampling space is a measurable space $(\samplingSet, \Borel(\tau_\samplingSet))$, 
where $(\samplingSet,\tau_\samplingSet)$ is a topological space
and $\Borel(\tau)$ is the Borel $\sigma$-algebra generated by a topology $\tau$.
For the first sampling stage, the input space is then $(\distributions(\samplingSet),\Borel(\tau_w))$, 
where $\distributions(\samplingSet)$ is the set of Borel probability measures on $\samplingSet$, 
and $\tau_w$ the topology of weak convergence in $\distributions(\samplingSet)$.
The data-generating distribution, also called \defm{meta-distribution} in this context, is now a probability measure $P$ on $\distributions(\samplingSet)\times\outputSet$.
A data set $\dataSet$ with $N\in\Np$ data points is generated as follows.
In the first stage, a data set
\begin{equation} \label{eq:distrslt:dataSetBar}
    \dataSetBar=((Q_1,y_1),\ldots,(Q_N,y_N)) \in (\distributions(\samplingSet)\times\outputSet)^N
\end{equation}
is generated by $(Q_1,y_1),\ldots,(Q_N,y_N) \iiddistr P$.
In the second stage, given $M^{(1)},\ldots,M^{(N)}\in\Np$, we sample independently
\begin{equation*}
    S^{(n)}_1,\ldots,S^{(n)}_{M^{(n)}} \distr Q_n, \quad n=1,\ldots,N
\end{equation*}
and then set
\begin{equation} \label{eq:distrslt:dataSet}
    \dataSet = \big((S^{(1)},y_1),\ldots,(S^{(N)},y_N)\big) \in (\samplingSet^\ast\times\outputSet)^N,
\end{equation}
where we defined 
\ifNotMuchSpace
$ \samplingSet^\ast = \bigcup_{M\in\Np} \samplingSet^M$
\else
\begin{equation*}
    \samplingSet^\ast = \bigcup_{M\in\Np} \samplingSet^M
\end{equation*}
\fi
and $S^{(n)}=\left(S^{(n)}_1,\ldots,S^{(n)}_{M^{(n)}}\right)$, for $n=1,\ldots,N$.
\paragraph{Hilbertian embeddings}
We now outline the use of Hilbertian embeddings in this context, following the axiomatic approach from \cite{fiedler2024statistical}.
A \defm{Hilbertian embedding} is a map  $\hilbertianEmbed: \distributions(\samplingSet) \rightarrow \hilbertianHS$,
where the \defm{embedding space } $\hilbertianHS$ is a (real) Hilbert space,
inducing a new input space $\inputSet=\hilbertianEmbed(\distributions(\samplingSet))\subseteq\hilbertianHS$.
We also assume access to \defm{embedding estimators} $(\hat\hilbertianEmbed_M)_{M\in\Np}$, $\hat\hilbertianEmbed_M: \samplingSet^M \rightarrow \hilbertianHS$,
and define $\hat\hilbertianEmbed: \samplingSet^\ast \rightarrow \inputSet$  by $\hat\hilbertianEmbed(S)=\hat\hilbertianEmbed_M(S)$ for all $S\in\samplingSet^M$ and $M\in\Np$.
Furthermore, for a first-stage data set $\dataSetBar$ from \eqref{eq:distrslt:dataSetBar}, we define
\ifNotMuchSpace
$\dataSetBar_\hilbertianEmbed = \big((\hilbertianEmbed (Q_n), y_n)\big)_{n=1,\ldots,N} \in (\inputSet\times\outputSet)^N$,
\else
\begin{equation*}
     \dataSetBar_\hilbertianEmbed = \left((\hilbertianEmbed (Q_n), y_n)\right)_{n=1,\ldots,N} \in (\inputSet\times\outputSet)^N,
\end{equation*}
\fi
and for a second-stage data set $\dataSet$ from \eqref{eq:distrslt:dataSet}, we define
\ifNotMuchSpace
$\dataSet_{\hat\hilbertianEmbed} = \big((\hat\hilbertianEmbed(S^{(n)}),y_n)\big)_{n=1,\ldots,N}$.
\else
\begin{equation*}
     \dataSet_{\hat\hilbertianEmbed} = \left((\hat\hilbertianEmbed(S^{(n)}),y_n)\right)_{n=1,\ldots,N} \in (\inputSet\times\outputSet)^N.
\end{equation*}
\fi
To avoid measurability issues, one can use the following assumption.
\begin{assumption}\label{assumption:distrslt:sufficientConditionMeasurability}
$\hilbertianHS$ is separable, 
$\hilbertianEmbed$ is $\Borel(\tau_{w})$-$\Borel(\hilbertianHS)$-measurable, 
and $\inputSet\in\Borel(\hilbertianHS)$.
Furthermore, for all $M\in\Np$, $\hat\hilbertianEmbed_M$ is $\Borel(\tau_\samplingSet)^{\otimes M}$-$\Borel(\inputSet)$-measurable.
\end{assumption}
The following result then takes care of measurability issues.
\begin{lemma} \label{lem:distrslt:measurability}
Under Assumption \ref{assumption:distrslt:sufficientConditionMeasurability}, the map $\hilbertianEmbed$ is $\Borel(\tau_w)$-$\Borel(\tau_\hilbertianHS\lvert_{\inputSet})$-measurable, where $\tau_\hilbertianHS\lvert_{\inputSet}$ is the subspace topology on $\inputSet$ induced by the topology on $\hilbertianHS$.
Furthermore, every $P\in\distributions(\distributions(\samplingSet)\times\outputSet)$ induces a distribution $P_\hilbertianEmbed$ on $\inputSet\times\outputSet$ 
as the pushforward
of $P$ along $(Q,y)\mapsto(\hilbertianEmbed (Q), y)$.
\end{lemma}
A proof of this result is provided in Section A.1.1 in \cite{szabo2015two} and the supplementary to \cite{lopezpaz2015towards}.
For the analysis later on, we need probabilistic estimation bounds for the Hilbertian embeddings, which we abstract in the next assumption.
\begin{assumption} \label{assumption:distrslt:embeddingsEstimationBounds}
We have access to $\hilbertianEmbedBound: \Np\times(0,1)\rightarrow\Rnn$ such that for all $Q\in\distributions(\samplingSet)$, $M\in\Np$, $\delta\in(0,1)$
\begin{equation*}
    \Pb_{S \distr Q^{\otimes M}}[\|\hilbertianEmbed(Q)-\hat\hilbertianEmbed(S)\|_\hilbertianHS > \hilbertianEmbedBound(M,\delta)]<\delta
\end{equation*}
\end{assumption}
We would like to stress the embedding strategy outlined here requires a kernel $k$ on $\inputSet$, which is a subset of an in general infinite-dimensional Hilbert space.
The availability of such kernels is one major advantage of this approach, and we refer to \cite{meunier2022distribution} for more details.
In the next example, we recall an important instance of such a kernel.
\begin{example} \label{example:distrslt:gaussianKernelHS}
Let $\emptyset\not=X\subseteq\hs{H}$ be a subset of an arbitrary real Hilbert space.
For $\gamma\in\Rp$,
\ifNotMuchSpace
$\gaussianKernel{\gamma}(x,x')=\exp\left(-\|x-x'\|_\hs{H}^2/\gamma^2\right)$
\else
\begin{equation*}
    \gaussianKernel{\gamma}(x,x')=\exp\left(-\frac{\|x-x'\|_\hs{H}^2}{\gamma^2}\right)
\end{equation*}
\fi
is a kernel on $X$, called \defm{Gaussian kernel}, and we denote its unique RKHS by $(H_\gamma,\|\cdot\|_{k_\gamma})$,
cf. \cite{christmann2010universal} and \cite{meunier2022distribution} for more details.
\end{example}
In general, in the following we will work with kernels $k$ that fulfill the next (rather mild) assumption.
For instance, since $\inputSet\subseteq \hilbertianHS$ is separable, cf. Assumption \ref{assumption:distrslt:sufficientConditionMeasurability}, the Gaussian kernel from Example \ref{example:distrslt:gaussianKernelHS} fulfills it.
\begin{assumption} \label{assumption:distrslt:kernelForSVM}
The kernel $k$ on $\inputSet$ is measurable, bounded, and has a separable RKHS $H_k$.
Furthermore, there exists $\alpha_k \in \cK$ such that
\ifNotMuchSpace
$  \|\fm_\kTwo(x)-\fm_\kTwo(x')\|_\kTwo \leq \alpha_\kTwo(\|x-x'\|_\hilbertianHS)$
\else
\begin{equation*}
      \|\fm_\kTwo(x)-\fm_\kTwo(x')\|_\kTwo \leq \alpha_\kTwo(\|x-x'\|_\hilbertianHS)
\end{equation*}
\fi
holds for all $x,x'\in\inputSet$.
\end{assumption}
The last condition in the preceding assumption is easily fulfilled for Hölder-continuous kernels, cf. \cite{fiedler2023lipschitz} for a thorough discussion of this aspect, and \cite{szabo2015two} for an extensive list of concrete examples of such kernels.

As a concrete example of Hilbertian embeddings, we use \defm{kernel mean embeddings} (KMEs).
For the reader's convenience, we collect now some well-known definitions and facts, cf. \cite{lopezpaz2015towards,szabo2016learning} for more details and proofs.
\begin{proposition} \label{prop:distrslt:kmeBackground}
Let $(\samplingSet,\setsys{A}_\samplingSet)$ be a measurable space, and $\kKME$ a measurable and bounded kernel on $\samplingSet$ with separable RKHS $H_\kKME$.
(i) The map
\begin{equation}
    \KMEk: \distributions(\samplingSet) \rightarrow H_\kKME, \:
    \KMEk Q = \int \kKME(\cdot,s)\mathrm{d}Q(s)
\end{equation}
is well-defined, and we call $\KMEk Q$ the \emph{kernel mean embedding (KME)} of $Q\in\distributions(\samplingSet)$ w.r.t. $\kKME$.

(ii) Define $\KMEhat: \samplingSet^\ast \rightarrow H_\kKME$ by
\begin{equation}
    \KMEhat((s_1,\ldots,s_M)) = \frac1M \sum_{m=1}^M \kKME(\cdot,s_m).
\end{equation}
For all $Q\in \distributions(\samplingSet)$ and $S\distr Q^{\otimes M}$, $M\in\Np$, and $\delta\in(0,1)$, we have that
\begin{equation}
    \|\KMEhat S - \KMEk Q\|_\kKME \leq 2\sqrt{\frac{\|\kKME\|_\infty^2}{M}} + \sqrt{\frac{2\|\kKME\|_\infty \ln(1/\delta)}{M}}
\end{equation}
holds with probability at least $1-\delta$.

(iii) Let $(\samplingSet,\tau_\samplingSet)$ be a separable topological space, consider $\setsys{A}_\samplingSet=\Borel(\tau_\samplingSet)$, and assume that $\kKME$ is continuous, then $\KMEk$ is $(\distributions(\samplingSet), \Borel(\tau_w))$-$(H_\kKME, \Borel(H_\kKME))$-measurable.
\end{proposition}

\section{Consistency and Learning Rates} \label{sec:distrslt:consistencyAndLearningRates}
We will now present our first main results.
Building on a rather general oracle inequality stated in Section \ref{sec:distrslt:oracleInequality},
we establish (universal) consistency for SVMs in the two-stage sampling setup in Section \ref{sec:distrslt:consistency},
and then learning rates in Section \ref{sec:distrslt:learningRates}.
\subsection{Oracle Inequality} \label{sec:distrslt:oracleInequality}
As common in statistical learning theory, consistency and learning rates can be derived from an oracle inequality.
The following result will be our central tool for this task, and it is a two-stage sampling variant of \cite[Theorem~7.22]{SC08}.
\begin{theorem} \label{thm:distrslt:oracleInequSVM}
Consider the two-stage sampling setup outlined in Section \ref{sec:distrslt:setup}.
Let $k$ be a kernel on $\inputSet$ that fulfills Assumption \ref{assumption:distrslt:kernelForSVM},
let the Hilbertian embedding fulfill Assumptions \ref{assumption:distrslt:sufficientConditionMeasurability} and \ref{assumption:distrslt:embeddingsEstimationBounds},
and consider a convex, locally Lipschitz-continuous loss $\ell$ that can be clipped at $\clipValue\in\Rp$.
Finally, assume that there exists $B\in\Rnn$ such that
\begin{equation}
    \ell(y,t) \leq B \quad \forall y\in \outputSet, t\in[-\clipValue,\clipValue]
\end{equation}
holds.
Then there exists a universal constant $C\in\Rp$ such that for all $N\geq 2$, $\lambda\in\Rp$ and $\tau\geq 1$ it holds with probability at least $1-4e^{-\tau}$ that
\ifTwoColumns
\begin{align} \label{eq:distrslt:oracleInequSVM}
    & \Risk{\ell}{P_\hilbertianEmbed}(\clipped{f}_{\dataSet_{\hat\hilbertianEmbed},\lambda})
        + \lambda \|f_{\dataSet_{\hat\hilbertianEmbed},\lambda}\|_k^2 - \RiskBayes{\ell}{P_\hilbertianEmbed}
        \leq 9 \Aef(\lambda)  \\ \nonumber
    & \hspace{0.5cm} + 9(\Risk{\ell}{P_{\hilbertianEmbed}}^{H_k\ast} - \RiskBayes{\ell}{P_\hilbertianEmbed}) 
        + C \LocLip{\ell}{M}^2\|k\|_\infty \frac{\ln N}{N\lambda} \\ \nonumber 
    & \hspace{0.5cm} + 300 \frac{B\tau}{\sqrt{N}} 
        + 15\frac{\tau}{N}\LocLip{\ell}{C_\lambda}\|k\|_\infty\sqrt{\frac{\Aef(\lambda)}{\lambda}} \\ \nonumber
    & \hspace{0.5cm} + \frac{3}{N}\sum_{n=1}^N \alpha_\lambda\left(\hilbertianEmbedBound(M^{(n)},e^{-\tau}/N)\right),
\end{align}
\else
\begin{align}  \label{eq:distrslt:oracleInequSVM}
        & \Risk{\ell}{P_\hilbertianEmbed}(\clipped{f}_{\dataSet_{\hat\hilbertianEmbed},\lambda})
        + \lambda \|f_{\dataSet_{\hat\hilbertianEmbed},\lambda}\|_k^2 - \RiskBayes{\ell}{P_\hilbertianEmbed}
    \leq
        9 \Aef(\lambda) + 9(\Risk{\ell}{P_{\hilbertianEmbed}}^{H_k\ast} - \RiskBayes{\ell}{P_\hilbertianEmbed}) 
        + C \LocLip{\ell}{M}^2\|k\|_\infty \frac{\ln N}{N\lambda} + 300 \frac{B\tau}{\sqrt{N}} \\ \nonumber
    & \hspace{1cm} + 15\frac{\tau}{N}\LocLip{\ell}{C_\lambda}\|k\|_\infty\sqrt{\frac{\Aef(\lambda)}{\lambda}} 
        + \frac{3}{N}\sum_{n=1}^N \alpha_\lambda\left(\hilbertianEmbedBound(M^{(n)},e^{-\tau}/N)\right),
\end{align}
\fi
where
\begin{align*}
    \alpha_\lambda = \left(
        \LocLip{\ell}{C_\lambda}\sqrt{\Aef(\lambda)/\lambda}
        + \LocLip{\ell}{M}\sqrt{B/\lambda}
        \right)\alpha_k.
\end{align*}
and
\begin{equation*}
    C_\lambda = \|k\|_\infty\sqrt{\Aef(\lambda)/\lambda}.
\end{equation*}
\end{theorem}
On a high level, for the proof we use continuity properties to go from the accessible data set $\dataSet_{\hat\hilbertianEmbed}$ to the inaccessible first-stage sampling data set $\dataSetBar_\hilbertianEmbed$, on which existing results can be applied.
While this is the standard strategy for the two-stage sampling setup, going back to \cite{szabo2015two,lopezpaz2015towards} and also used by \cite{fiedler2024statistical}, we use a rather advanced oracle inequality for the first stage of sampling, which requires some work.
Similarly as in the proof of \cite[Theorem~7.22]{SC08},
we first establish an oracle inequality for general approximate RERM schemes, 
and then check that SVMs indeed fulfill the necessary assumptions for this class of learning methods.
\ifCameraReady
A detailed proof is provided in the supplementary material.
\fi
\subsection{Consistency} \label{sec:distrslt:consistency}
We now turn to consistency of SVMs in the two-stage sampling setup.
It is clear that for consistency to hold, the hypothesis space $H_k$ has to be expressive enough.
This is formalized in the next assumption.
\begin{assumption} \label{assumption:distrslt:noApproxError}
    It holds that $\RiskOpt{\ell}{P_\hilbertianEmbed}{H_k}=\RiskBayes{\ell}{P}$.
\end{assumption}
For simplicity, we consider from now on only globally Lipschitz-continuous loss functions, including in particular the hinge loss $\ellHinge$, which is of prime importance for classification.
We are now ready to state the following general consistency result for SVMs with distributional inputs in the two-stage sampling setup.
\begin{proposition} \label{prop:distrslt:consistencySVM}
Consider the situation of Theorem \ref{thm:distrslt:oracleInequSVM}.
Assume that the loss function $\ell$ is globally $L_\ell$-Lipschitz continuous, and 
assume that for a data set of size $N\in\Np$, for every data point, $M_N\in\Np$ samples are drawn in the second stage of sampling, so $M^{(1)}=\ldots=M^{(N)}=M_N$ for $N$ samples.
If Assumption \ref{assumption:distrslt:noApproxError} holds, 
and if $(\lambda_N)_N\subseteq\Rp$ and $(M_N)_N\subseteq\Np$ are sequences such that $\lim_{N\rightarrow\infty} \lambda_N = 0$ and
\ifTwoColumns
\begin{equation} \label{eq:distrslt:consistencySVM:conditionsForSequences}
    \lim_{N\rightarrow\infty} \frac{\ln(N)}{N\lambda_N} 
    = \lim_{N\rightarrow\infty} \frac{1}{\sqrt{\lambda_N}} \alpha_k(\hilbertianEmbedBound(M_N,1/N)) = 0,
\end{equation}
\else
\begin{equation} \label{eq:distrslt:consistencySVM:conditionsForSequences}
    \lim_{N\rightarrow\infty} \frac{\ln(N)}{N\lambda_N} = 0
    \quad \text{and} \quad
    \lim_{N\rightarrow\infty} \frac{1}{\sqrt{\lambda_N}} \alpha_k(\hilbertianEmbedBound(M_N,1/N)) = 0,
\end{equation}
\fi
then 
\begin{equation*}
    (\samplingSet^{M_N}\times \outputSet)^N \ni \dataSet^{(N)}\mapsto \clipped{f}_{\dataSet^{(N)}_{\hat\hilbertianEmbed},\lambda_N} \circ \hilbertianEmbed
\end{equation*}
is an $\ell$-risk consistent learning method, so
\begin{equation*}
    \Risk{\ell}{P}(\clipped{f}_{\dataSet^{(N)}_{\hat\hilbertianEmbed},\lambda_N} \circ \hilbertianEmbed) \rightarrow \RiskBayes{\ell}{P}
\end{equation*}
in probability for $N\rightarrow\infty$, for all data-generating distributions $P$ under Assumption \ref{assumption:distrslt:noApproxError}.
\end{proposition}
This result poses two conditions on $(\lambda_N)_N$ and $(M_N)_N$.
The first one appears also in the usual statistical learning setup with only a single stage of sampling, cf. the discussion in \cite[Section~7.4]{SC08}.
The second condition arises through the two-stage sampling setup, which is a well-known effect in the case of distributional regression, cf. \cite{szabo2015two,szabo2016learning}.
\begin{proof}
Observe that 
\begin{align*}
   \RiskBayes{\ell}{P} & =\inf_{f\text{ measurable}} \Risk{\ell}{P}(f) \\
   & \leq \inf_{f\text{ measurable}} \Risk{\ell}{P}(f \circ \hilbertianEmbed) =\RiskBayes{\ell}{P_\hilbertianEmbed}, 
\end{align*}
and
$\Risk{\ell}{P}(\clipped{f}_{\dataSet_{\hat\hilbertianEmbed},\lambda} \circ \hilbertianEmbed) - \RiskBayes{\ell}{P}
    \leq 
    \Risk{\ell}{P_\hilbertianEmbed}(\clipped{f}_{\dataSet_{\hat\hilbertianEmbed},\lambda})
    + \lambda \|\clipped{f}_{\dataSet_{\hat\hilbertianEmbed},\lambda}\|_k^2 - \RiskBayes{\ell}{P_\hilbertianEmbed}$.
By definition of $\ell$-risk consistency and using Assumption \ref{assumption:distrslt:noApproxError}, it is hence enough to ensure that for fixed $\tau\geq 1$, the righthand side in \eqref{eq:distrslt:oracleInequSVM} converges to zero for $N\rightarrow \infty$.
Since $\lambda_N\rightarrow 0$, we have $\Aef(\lambda_N)\rightarrow 0$, cf. \cite[Lemma~5.15]{SC08}.
Furthermore, $\frac{\ln(N)}{N\lambda_N}\rightarrow 0$ implies $N\lambda_N\rightarrow \infty$, %
so $15\frac{\tau}{N}L_\ell \|k\|_\infty\sqrt{\Aef(\lambda_N)/\lambda_N} = 15\tau L_\ell \|k\|_\infty / \sqrt{N} \times \sqrt{\Aef(\lambda_N)/(N\lambda_N)}\rightarrow 0$.
Finally, the last condition in \eqref{eq:distrslt:consistencySVM:conditionsForSequences} ensures that also the remaining terms converge to zero.
Altogether, this shows that the righthand side in \eqref{eq:distrslt:oracleInequSVM} indeed converges to zero, establishing consistency.
\end{proof}
For concreteness we now consider KMEs as Hilbertian embeddings and we assume Hölder-continuity of $\fm_k$.
In this situation, we can achieve the following consistency result.
\begin{proposition} \label{prop:distrslt:consistencySVMwithKMEs}
Consider the situation of Theorem \ref{thm:distrslt:oracleInequSVM} and assume that the loss function $\ell$ is globally $L_\ell$-Lipschitz continuous.
Assume that for a data set of size $N\in\Np$, for every data point, $M_N\in\Np$ samples are drawn in the second stage of sampling, and that KMEs (from Proposition \ref{prop:distrslt:kmeBackground}) are used for the Hilbertian embedding.
Furthermore, assume that there exist $C_k,\alpha\in\Rp$ such that $\alpha_k(s)=C_k s^\alpha$.
If Assumption \ref{assumption:distrslt:noApproxError} holds, 
and if $(\lambda_N)_N\subseteq\Rp$ and $(M_N)_N\subseteq\Np$ are sequences such that $\lim_{N\rightarrow\infty} \lambda_N = 0$ and
\begin{equation}
    \lim_{N\rightarrow\infty} \frac{\ln(N)}{N\lambda_N} = 0
    \quad \text{and} \quad
    \lim_{N\rightarrow\infty}\frac{\ln(N)^\alpha}{\lambda_N M_N^\alpha} = 0,
\end{equation}
then $(\samplingSet^{M_N}\times \outputSet)^N \ni \dataSet^{(N)}\mapsto \clipped{f}_{\dataSet^{(N)}_{\KMEhat},\lambda_N} \circ \KMEk$
is a  $\ell$-risk consistent learning method for all data-generating distributions $P$ under Assumption \ref{assumption:distrslt:noApproxError}. 
\end{proposition}
This result follows as an immediate corollary from Proposition \ref{prop:distrslt:consistencySVM}, and we provide a detailed proof 
\ifCameraReady
in the supplementary material.
\fi

\subsection{Learning Rates} \label{sec:distrslt:learningRates}
As is well-known, learning rates can only be derived under distributional assumptions due to the \defm{No Free Lunch Theorem}.
The form of the oracle inequality in \Cref{thm:distrslt:oracleInequSVM} shows that any distributional assumption must enter through the approximation error function.
The following is a standard assumption for this task, cf. \cite[Chapter~6]{SC08} and \cite{steinwart2009optimal}.
\begin{assumption} \label{assumption:distrslt:decayApproxErrorFunc}
There exist constants $C_\mathcal{A} \in\Rp$, $\beta\in(0,1]$ such that $\Aef(\lambda) \leq C_\mathcal{A}\lambda^\beta$ for all $\lambda\in\Rp$. 
\end{assumption}
For concreteness, we use this to establish a learning rate in the case of KMEs for the Hilbertian embeddings.
Learning rates for other embeddings can be derived similarly, and 
\ifCameraReady
in the supplementary material we provide a corresponding result for generic Hilbertian embeddings.
\fi
\begin{theorem} \label{thm:distrslt:learningRateSVMwithKMEs}
Consider the situation of Proposition \ref{prop:distrslt:consistencySVMwithKMEs},
assume that $\alpha\in(0,2]$,
and let in addition Assumption \ref{assumption:distrslt:decayApproxErrorFunc} hold for $\hilbertianEmbed=\KMEk$ from Proposition \ref{prop:distrslt:kmeBackground}.
If $(M_N)_N$ grows at least as $N^{\frac{2}{\alpha}}$, and $(\lambda_N)_N$ decays as $N^{-\frac{1}{\beta+1}}$, then a learning rate of $\ln(N) N^{-\frac{\beta}{\beta+1}}$ is achieved.
\end{theorem}
This result can be derived from \Cref{thm:distrslt:oracleInequSVM} 
\ifCameraReady
using well-known elementary arguments.
\fi
\section{Classfication With Gaussian Kernels and the Hinge Loss} \label{sec:distrslt:classificationWithGaussianKernelAndHingeLoss}
We now turn to classification ($\outputSet=\{-1,1\}$) using the hinge loss $\ellHinge$ and the Gaussian kernel $k_\gamma$ on $\inputSet$, cf. Example \ref{example:distrslt:gaussianKernelHS}.
The theory in Section \ref{sec:distrslt:consistencyAndLearningRates} covers this case already, however, the important Assumption \ref{assumption:distrslt:decayApproxErrorFunc} is in general difficult to interpret.
For binary classification, margin and noise exponent assumptions are more intuitive, cf. \cite[Chapter~8]{SC08} for an overview,
and our goal in this section is to realize this also for distributional classification in the two-stage sampling setup.
For this, we will first introduce a new feature space for Gaussian kernels on (subsets of) Hilbert spaces, which is of independent interest,
and then establish a bound on the approximation error function using a variant of an establish geometric margin condition.
\subsection{A New Feature Space}
In order to bound the approximation error function for the hinge loss and the Gaussian kernel, we need a convenient feature space.
For the Gaussian kernel on $X\subseteq\R^d$, one can use $L^2(X,\lambda_X^{(d)},\R)$, where $\lambda_X^{(d)}$ is the Lebesgue measure on $X$, cf. \cite[Section~4.4]{SC08}.
However, since we consider the Gaussian kernel on $\inputSet\subseteq\hilbertianHS$, where $\hilbertianHS$ is in general infinite-dimensional, we do not have the Lebesgue measure available anymore.
Instead, as common in infinite-dimensional analysis \cite{daprato2006introduction}, we use the Gaussian measure on a Hilbert space instead.
The next result describes how exactly this can be used to build a feature space-feature map pair for the Gaussian kernel on a separable Hilbert space.
\begin{theorem} \label{thm:distrslt:featureSpaceGaussianKernelOnHS}
Let $\hs{H}$ be a separable real Hilbert space, $\emptyset\not=X\subseteq\hs{H}$, and $\gaussianKernel{\gamma}$ the Gaussian kernel on $X$ with length scale $\gamma\in\Rp$.
For all $Q\in L_1^+(\hs{H})$ with $\ker(Q)=\{0\}$, $L^2_\R(\hs{H},\Normal(0,Q);\C)$ is a (real) feature space of $k_\gamma$,
\ifTwoColumns
$\fm_Q: X \rightarrow L^2_\R(\hs{H},\Normal(0,Q);\C)$ defined by
\begin{equation}
    \fm_Q(x) = \exp(i \sqrt{2}/\gamma \cdot W_x(\cdot))
\end{equation}
\else
\begin{equation}
    \fm_Q: X \rightarrow L^2_\R(\hs{H},\Normal(0,Q);\C),
    \quad
    \fm_Q(x) = \exp(i \sqrt{2}/\gamma \cdot W_x(\cdot))
\end{equation}
\fi
is a corresponding feature map, 
and the canonical surjection $V_Q: L^2_\R(\hs{H},\Normal(0,Q);\C) \rightarrow H_\gamma$ is given by
\begin{equation}
    (V_Qg)(x) = \Re \int_\hs{H} \exp\left(-i \frac{\sqrt{2}}{\gamma} W_x(z)\right)g(z)\mathrm{d}\Normal(z\mid 0, Q).
\end{equation}
\end{theorem}
\ifCameraReady
We provide a detailed proof of this result in the supplementary material.
\fi
\subsection{Learning Rates} \label{sec:distrslt:learningRatesClassification}
In order to bound the approximation error function for the hinge loss, we follow the high level strategy from \cite{steinwart2007fast}.
First, we need some preliminaries.
We consider a distribution $P$ on $\inputSet\times\outputSet$, and recall that $\outputSet=\{-1,1\}$ and $\inputSet\subseteq\hilbertianHS$ (later on, $P_\hilbertianEmbed$ will play the role of this $P$).
Let $\eta: \inputSet\rightarrow[0,1]$ be a version of the conditional probability $\Pb_{(X,Y)\distr P}[Y=1 \mid X=x]$,
and define
\begin{equation*}
    X_1=\{x\in\inputSet \mid \eta(x) > \frac12 \}
    \quad
    X_{-1}=\{x\in\inputSet \mid \eta(x)<\frac12\}
\end{equation*}
and
\begin{equation*}
    \Delta(x) = \begin{cases}
        d_\hilbertianHS(x, X_1) & \text{if } x \in X_{-1} \\  %
        d_\hilbertianHS(x, X_{-1}) & \text{if } x \in X_1 \\ %
        0 & \text{otherwise}
    \end{cases}
\end{equation*}
where $d_\hilbertianHS(x,A)=\inf_{x'\in A} \|x-x'\|_\hilbertianHS$ for $x\in\hilbertianHS$ and $A\subseteq\hilbertianHS$.
Furthermore, define $f_P^\ast: \inputSet\rightarrow[-1,1]$ as
\begin{equation*}
    f_P^\ast(x) = \begin{cases}
        1 & \text{if } x \in X_1 \\
        -1 & \text{if } x \in X_{-1} \\
        0 & \text{otherwise}
    \end{cases}
\end{equation*}
Then $f_P^\ast$ is measurable and achieves the Bayes risk, i.e., $\Risk{\ellClass}{P}(f_P^\ast)=\RiskBayes{\ellClass}{P}$.

We will now state the central assumption of this section.
It can be interpreted as a variant of the geometric noise exponent assumption from \cite{steinwart2007fast}, adapted to a Hilbert space setting.
\begin{assumption} \label{assumption:distrslt:geometricNoiseAssumptionGaussian}
$P$ and $\eta$ are such that there exist $Q\in L_1^+(\hs{H})$ with $\ker(Q)=\{0\}$ and constants $C_Q,\alpha_Q,\bar{t}_Q\in\Rp$ such that
for all $0<t\leq \bar{t}_Q$ it holds that
\ifTwoColumns
{\small
\begin{align} \label{eq:assumption:distrslt:geometricNoiseAssumptionGaussian:1}
    & \int_{X_1 \cup X_{-1}} \left(
            1- 2\int_{B_{\Delta(x)}(x)}
            \exp\left(-\frac{\|x-y\|_\hs{H}^2}{t}\right)
            \mathrm{d} \Normal(y\mid 0,Q)
        \right) \nonumber \\
    & \hspace{0.5cm} \times |2\eta(x)-1|\mathrm{d}P_X(x) \leq C_Q t^{\alpha_Q}.
\end{align}
}
\else
\begin{equation} \label{eq:assumption:distrslt:geometricNoiseAssumptionGaussian:1}
    \int_{X_1 \cup X_{-1}} \left(
            1- 2\int_{B_{\Delta(x)}(x)}
            \exp\left(-\frac{\|x-y\|_\hs{H}^2}{t}\right)
            \mathrm{d} \Normal(y\mid 0,Q)
        \right)  |2\eta(x)-1|\mathrm{d}P_X(x)
        \leq C_Q t^{\alpha_Q}.
\end{equation}
\fi
and
\ifTwoColumns
\begin{align} \label{eq:assumption:distrslt:geometricNoiseAssumptionGaussian:2}
    & \int_\inputSet \left(
            \int_\hs{H} \exp\left(-\frac1t \|y\|_\hs{H}^2\right) \mathrm{d} \Normal(y\mid x,Q)
        \right) \nonumber \\
    & \hspace{0.5cm} \times |2\eta(x)-1|\mathrm{d}P_X(x) \leq C_Q t^{\alpha_Q}
\end{align}
\else
\begin{equation} \label{eq:assumption:distrslt:geometricNoiseAssumptionGaussian:2}
    \int_\inputSet \left(
            \int_\hs{H} \exp\left(-\frac1t \|y\|_\hs{H}^2\right) \mathrm{d} \Normal(y\mid x,Q)
        \right)|2\eta(x)-1|\mathrm{d}P_X(x)
    \leq C_Q t^{\alpha_Q}
\end{equation}
\fi
\end{assumption}
While rather technical, the preceding assumption has a clear intuitive interpretation, which we discuss 
\ifCameraReady
in the supplementary material.
\fi
We are now ready to use this assumption to derive a bound on the approximation error function for Gaussian kernels on separable Hilbert spaces.
\begin{theorem} \label{thm:distrslt:boundApproxErrorFuncGaussianKernelHingeLoss}
Let Assumption \ref{assumption:distrslt:geometricNoiseAssumptionGaussian} hold, then for all $\gamma\in\Rp$ with $\gamma^2<\bar t_Q$, we have
\begin{equation}
    \ApproxErrorFunc{\ellHinge}{P}{H_\gamma}(\lambda) \leq 2C_Q \gamma^{2\alpha_Q} + \lambda
\end{equation}
for all $\lambda\in\Rp$.
\end{theorem}
The proof uses the strategy from \cite[Section~4]{steinwart2007fast}, cf. also \cite[Section~8.2]{SC08}:
An explicit Bayes optimal classifier is embedded into the Gaussian RKHS (here via Theorem \ref{thm:distrslt:featureSpaceGaussianKernelOnHS}), which is interpreted as a smoothing, and the resulting performance degradation, i.e., increase in risk, is bounded using a margin or noise assumption, here Assumption \ref{assumption:distrslt:geometricNoiseAssumptionGaussian}.
A detailed proof is provided 
\ifCameraReady
in the supplementary material.
\fi

Finally, all of this can be combined to arrive at the following result on learning rates for distributional classification with hinge-loss SVMs with Gaussian kernels in the two-stage sampling setup.
For concreteness, we consider KMEs for the Hilbertian embeddings, but analogous results can be derived in a similar manner for other embeddings.
\begin{theorem} \label{thm:distrslt:learningRateClassificationGaussianKernel}
Consider the situation of Proposition \ref{prop:distrslt:consistencySVMwithKMEs} with $\ell=\ellHinge$ and the Gaussian kernel on $\inputSet$,
and let in addition Assumption \ref{assumption:distrslt:geometricNoiseAssumptionGaussian} hold for $P=P_{\KMEk}$.
If $(M_N)_N\subseteq\Np$ grows at least as $N^{\frac{2}{\alpha}}$, 
$(\lambda_N)_N\subseteq\Rp$ decays as $N^{-\frac12}$, 
$(\gamma_N)_N\subseteq\Rp$ decays as $N^{-\mu}$ for some $\mu\in\Rp$, 
and Assumption \ref{assumption:distrslt:noApproxError} holds for all $H_{\gamma_N}$,
then $(\samplingSet^{M_N}\times \outputSet)^N \ni \dataSet^{(N)}\mapsto \clipped{f}_{\dataSet^{(N)}_{\KMEhat},\lambda_N}^{H_{\gamma_N}} \circ \KMEk$
achieves a learning rate of $N^{-\min\{2\mu\alpha_Q, \frac12\}}$.
\end{theorem}
\ifCameraReady
This result follows directly from Theorems \ref{thm:distrslt:oracleInequSVM} and \ref{thm:distrslt:boundApproxErrorFuncGaussianKernelHingeLoss}.
\fi
\section{Conclusion}
We considered kernel-based statistical learning with distributional inputs in the two-stage sampling setup, for which we established consistency and learning rates for rather general loss functions, covering the important case of binary classification.
Furthermore, using a novel variant of an establish geometric margin exponent assumption, we were able to prove learning rates for distributional classification with SVMs using the hinge loss and Gaussian kernels. 
In particular, we could establish a learning rate without relying on an explicit smoothness assumption as common in regression, since this can be inappropriate for a classification setting.
While we focused primarily on KMEs as Hilbertian embeddings, our results can be easily adapted to other embeddings.
Our work opens up a multitude of interesting directions.
First, by using a supremum bound in an oracle inequality like Theorem \ref{thm:distrslt:oracleInequSVM}, our rates can be further refined, as in the case of single-stage sampling, cf. \cite[Chapter~7]{SC08}.
Second, using appropriate discretizations, for example via entropy number estimates, is another avenue for refinement of our rates, which requires dealing with a delicate interplay of the embedding map, the (in general infinite-dimensional) embedding Hilbert space, and the kernel used in the SVM.
Third, a closer investigation of Assumption \ref{assumption:distrslt:geometricNoiseAssumptionGaussian} and potential refinements is another interesting aspect for future work.
Finally, extensions to related learning problems like multiclass classification are another line of interesting future work.
\clearpage

\onecolumn

\section*{Acknowledgements}
The author would like to thank Pierre-François Massiani, Oleksii Kachaiev, and Ingo Steinwart for very helpful discussions, Alessandro Scagliotti for a careful reading of the manuscript, Mattes Mollenhauer for insightful comments on Section \ref{sec:distrslt:classificationWithGaussianKernelAndHingeLoss}, and anonymous reviewers for helpful comments. 
The author acknowledges funding from DFG Project FO 767/10-2 (eBer-24-32734) ”Implicit Bias in Adversarial Training”.
  
\begin{center}
    \raisebox{0pt}{\includegraphics[height=1.5cm]{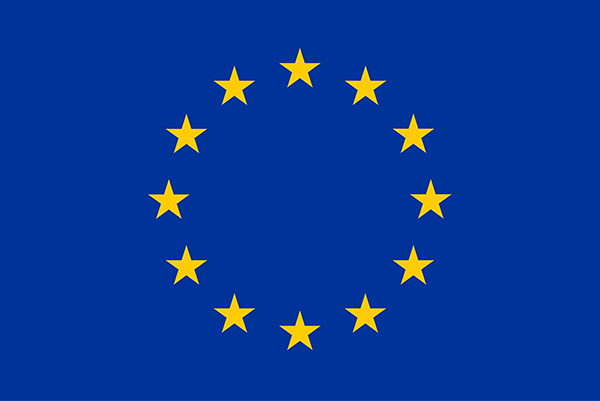}}%
    \hspace{1em}%
    \raisebox{0pt}{\includegraphics[height=1.5cm]{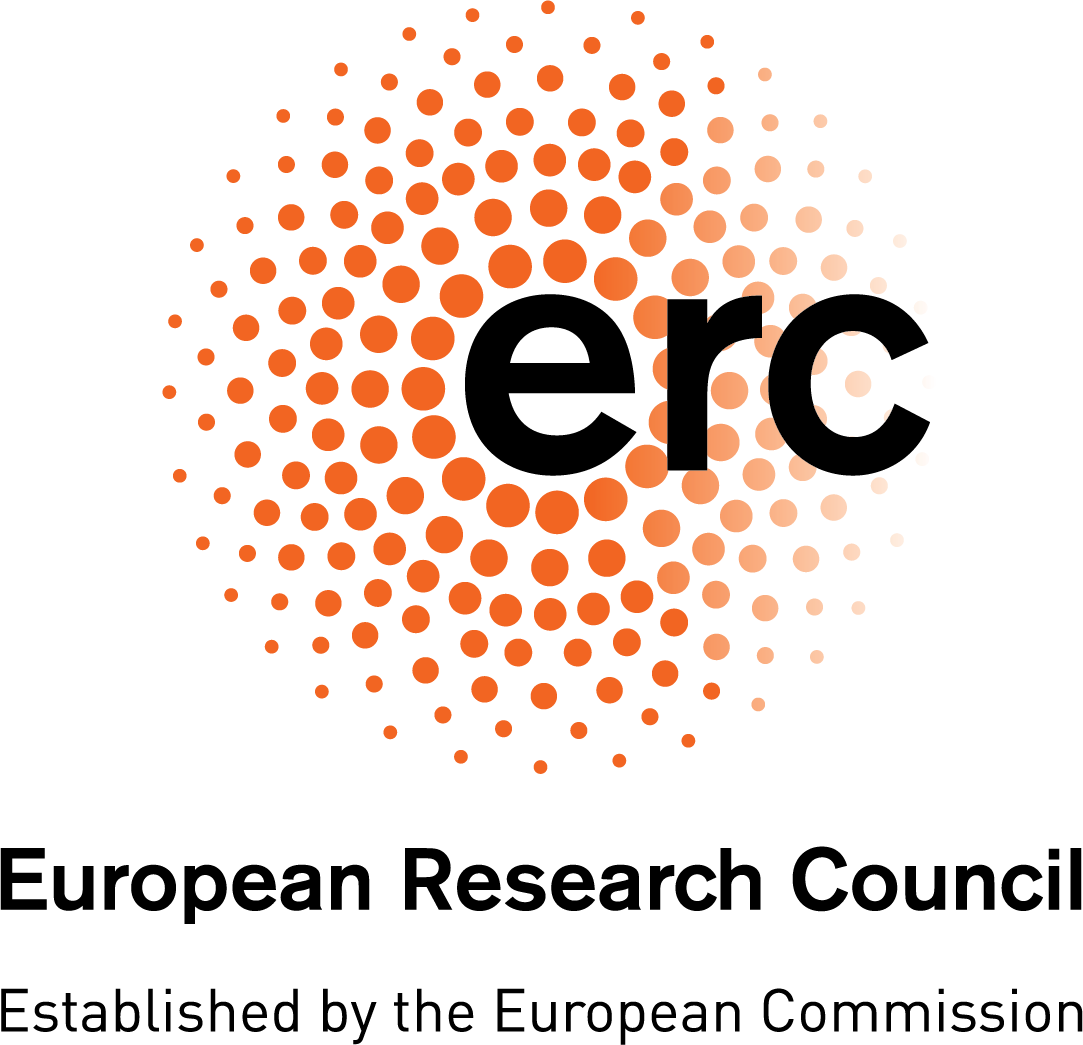}}%
\end{center}
Funded by the European Union. Views and opinions expressed are however those of the author(s) only and do not necessarily reflect those of the European Union or the European Research Council Executive Agency. Neither the European Union nor the granting authority can be held responsible for them. This project has received funding from the European Research Council (ERC) under the European Union’s Horizon Europe research and innovation programme (grant agreement No. 101198055, project acronym NEITALG).

\bibliography{refs}

\begin{thebibliography}{24}
\providecommand{\natexlab}[1]{#1}

\bibitem[{Bach(2024)}]{bach2024learning}
Bach, F. 2024.
\newblock \emph{Learning theory from first principles}.
\newblock MIT press.

\bibitem[{Bonnier, Oberhauser, and Szab{\'o}(2023)}]{bonnier2023kernelized}
Bonnier, P.; Oberhauser, H.; and Szab{\'o}, Z. 2023.
\newblock Kernelized Cumulants: Beyond Kernel Mean Embeddings.
\newblock \emph{Advances in Neural Information Processing Systems}.

\bibitem[{Caponnetto and De~Vito(2007)}]{caponnetto2007optimal}
Caponnetto, A.; and De~Vito, E. 2007.
\newblock Optimal rates for the regularized least-squares algorithm.
\newblock \emph{Foundations of Computational Mathematics}, 7: 331--368.

\bibitem[{Christmann and Steinwart(2010)}]{christmann2010universal}
Christmann, A.; and Steinwart, I. 2010.
\newblock Universal kernels on non-standard input spaces.
\newblock \emph{Advances in neural information processing systems}, 23.

\bibitem[{Da~Prato(2006)}]{daprato2006introduction}
Da~Prato, G. 2006.
\newblock \emph{An introduction to infinite-dimensional analysis}.
\newblock Springer Science \& Business Media.

\bibitem[{Da~Prato and Zabczyk(2002)}]{daprato2002second}
Da~Prato, G.; and Zabczyk, J. 2002.
\newblock \emph{Second order partial differential equations in Hilbert spaces}, volume 293.
\newblock Cambridge University Press.

\bibitem[{Fang, Guo, and Zhou(2020)}]{fang2020optimal}
Fang, Z.; Guo, Z.-C.; and Zhou, D.-X. 2020.
\newblock Optimal learning rates for distribution regression.
\newblock \emph{Journal of complexity}, 56: 101426.

\bibitem[{Fiedler(2023)}]{fiedler2023lipschitz}
Fiedler, C. 2023.
\newblock Lipschitz and {H}\"older Continuity in Reproducing Kernel Hilbert Spaces.
\newblock \emph{arXiv preprint arXiv:2310.18078}.

\bibitem[{Fiedler et~al.(2024)Fiedler, Massiani, Solowjow, and Trimpe}]{fiedler2024statistical}
Fiedler, C.; Massiani, P.-F.; Solowjow, F.; and Trimpe, S. 2024.
\newblock On statistical learning theory for distributional inputs.
\newblock In \emph{Forty-first International Conference on Machine Learning}.

\bibitem[{Kachaiev and Recanatesi(2024)}]{kachaiev2024learning}
Kachaiev, O.; and Recanatesi, S. 2024.
\newblock Learning to embed distributions via maximum kernel entropy.
\newblock \emph{Advances in Neural Information Processing Systems}, 37: 44710--44734.

\bibitem[{Kellett(2014)}]{kellett2014compendium}
Kellett, C.~M. 2014.
\newblock A compendium of comparison function results.
\newblock \emph{Mathematics of Control, Signals, and Systems}, 26: 339--374.

\bibitem[{Liu and Zhou(2025)}]{liu2025generalization}
Liu, P.; and Zhou, D.-X. 2025.
\newblock Generalization Analysis of Transformers in Distribution Regression.
\newblock \emph{Neural Computation}, 37(2): 260--293.

\bibitem[{Lopez-Paz et~al.(2015)Lopez-Paz, Muandet, Sch{\"o}lkopf, and Tolstikhin}]{lopezpaz2015towards}
Lopez-Paz, D.; Muandet, K.; Sch{\"o}lkopf, B.; and Tolstikhin, I. 2015.
\newblock Towards a learning theory of cause-effect inference.
\newblock In \emph{International Conference on Machine Learning}, 1452--1461. PMLR.

\bibitem[{Massiani et~al.(2025)Massiani, Haverbeck, Thesing et~al.}]{massiani2025robust}
Massiani, P.~F.; Haverbeck, L.; Thesing, C.; et~al. 2025.
\newblock Robust screening of atrial fibrillation with distribution classification.
\newblock \emph{Scientific Reports}, 15: 26582.

\bibitem[{Meunier, Pontil, and Ciliberto(2022)}]{meunier2022distribution}
Meunier, D.; Pontil, M.; and Ciliberto, C. 2022.
\newblock Distribution regression with sliced {W}asserstein kernels.
\newblock In \emph{International Conference on Machine Learning}, 15501--15523. PMLR.

\bibitem[{Muandet et~al.(2012)Muandet, Fukumizu, Dinuzzo, and Sch{\"o}lkopf}]{muandet2012learning}
Muandet, K.; Fukumizu, K.; Dinuzzo, F.; and Sch{\"o}lkopf, B. 2012.
\newblock Learning from distributions via support measure machines.
\newblock \emph{Advances in neural information processing systems}, 25.

\bibitem[{M{\"u}cke(2021)}]{mucke2021stochastic}
M{\"u}cke, N. 2021.
\newblock Stochastic gradient descent meets distribution regression.
\newblock In \emph{International Conference on Artificial Intelligence and Statistics}, 2143--2151. PMLR.

\bibitem[{P{\'o}czos et~al.(2013)P{\'o}czos, Singh, Rinaldo, and Wasserman}]{poczos2013distribution}
P{\'o}czos, B.; Singh, A.; Rinaldo, A.; and Wasserman, L. 2013.
\newblock Distribution-free distribution regression.
\newblock In \emph{artificial intelligence and statistics}, 507--515. PMLR.

\bibitem[{Steinwart and Christmann(2008)}]{SC08}
Steinwart, I.; and Christmann, A. 2008.
\newblock \emph{Support vector machines}.
\newblock Springer Science \& Business Media.

\bibitem[{Steinwart et~al.(2009)Steinwart, Hush, Scovel et~al.}]{steinwart2009optimal}
Steinwart, I.; Hush, D.~R.; Scovel, C.; et~al. 2009.
\newblock Optimal Rates for Regularized Least Squares Regression.
\newblock In \emph{COLT}, 79--93.

\bibitem[{Steinwart and Scovel(2007)}]{steinwart2007fast}
Steinwart, I.; and Scovel, C. 2007.
\newblock Fast rates for support vector machines using {G}aussian kernels.

\bibitem[{Szab{\'o} et~al.(2015)Szab{\'o}, Gretton, P{\'o}czos, and Sriperumbudur}]{szabo2015two}
Szab{\'o}, Z.; Gretton, A.; P{\'o}czos, B.; and Sriperumbudur, B. 2015.
\newblock Two-stage sampled learning theory on distributions.
\newblock In \emph{Artificial Intelligence and Statistics}, 948--957. PMLR.

\bibitem[{Szab{\'o} et~al.(2016)Szab{\'o}, Sriperumbudur, P{\'o}czos, and Gretton}]{szabo2016learning}
Szab{\'o}, Z.; Sriperumbudur, B.~K.; P{\'o}czos, B.; and Gretton, A. 2016.
\newblock Learning theory for distribution regression.
\newblock \emph{The Journal of Machine Learning Research}, 17(1): 5272--5311.

\bibitem[{Yu et~al.(2021)Yu, Ho, Shi, and Zhou}]{yu2021robust}
Yu, Z.; Ho, D.~W.; Shi, Z.; and Zhou, D.-X. 2021.
\newblock Robust kernel-based distribution regression.
\newblock \emph{Inverse Problems}, 37(10): 105014.

\end{thebibliography}

\appendix
\begin{center}
\Huge Supplementary Material
\end{center}

\section{Additional background}
\subsection{On surrogate losses, calibration, and assumptions for learning rates} \label{sec:distrslt:calibrationAndAssumptions}
In this section, we briefly review the concept of surrogate losses and calibration, and discuss some consequences for assumptions used to derive learning rates.

Consider the general setup of statistical learning theory outlined in Section \ref{sec:distrslt:prelims}.
The learning task, or rather the task that is to be learned by supervised learning, is described by a \defm{loss function} $\ell:\outputSet\times\R\rightarrow\Rnn$.
If a hypothesis $f: \inputSet\rightarrow\R$ achieves a small average loss (the risk) $\Risk{\ell}{P}(f)$, then it performs well on the task encoded by $\ell$.
For example, if we want to learn a binary classifier, i.e., a map that assigns a class to an object based on certain features $x\in\inputSet$ of the object, then this task is most naturally described by the zero-one loss $\ellClass$, cf. Example \ref{example:distrslt:classificationLoss}.
In this case, the risk is the probability of misclassification by the classifier.

However, the loss function should also allow an algorithmic implementation.
In the context of (regularized) empirical risk minimization, the loss determines (together with the hypothesis class and the regularizer) the properties of the resulting optimization problem, which should be amenable to efficient optimization algorithms.
From this perspective, the zero-one loss is particularly bad, since it is nonconvex and even non-continuous.
This motivates the use of a \defm{surrogate loss} $\tilde\ell:\outputSet\times\R\rightarrow\Rnn$ in the (regularized) empirical risk minimization.
The idea is that $\tilde\ell$ has better properties with regards to the final optimization problem, but still encodes the original task to be learned.
For classification, a typical choice of a surrogate loss is the hinge loss $\ellHinge$, cf. Example \ref{example:distrslt:hingeLoss}, which can be interpreted as a continuous and convex relaxation of the zero-one loss.

But in which sense does a surrogate loss maintain the original task?
In statistical learning theory, this is described by (loss)\defm{calibration}.
We give only a brief overview sufficient for the remainder of our discussion in this section, and refer to \cite[Chapters~2,3]{SC08} and \cite[Chapter~4]{bach2024learning} for more details.
The immediate goal of (empirical) risk minimization with the surrogate loss is the minimization of the risk (over the actual data-generating distribution) $\Risk{\tilde\ell}{P}$, or equivalently minimizing the excess risk $\Risk{\tilde\ell}{P}-\RiskBayes{\ell}{P}$.
To ensure that this leads to good performance relative to the actual task, encoded by $\ell$, we therefore need that a small excess risk w.r.t. the surrogate loss $\tilde\ell$ implies a small risk w.r.t. the original loss $\ell$.
One way to formalize this in a quantitative way is the ensure the existence of a function $\Upsilon: \Rnn \rightarrow\Rnn$ that is increasing, continuous in 0, and $\Upsilon(0)=0$, such that
\begin{equation*}
    \Risk{\ell}{P}(f)-\RiskBayes{\ell}{P} \leq \Upsilon\left(\Risk{\tilde\ell}{P}(f) - \RiskBayes{\tilde\ell}{P}\right)
    \quad \forall f: \inputSet\rightarrow\R \text{ measurable}
\end{equation*}
In words, if a hypothesis $f$ (e.g., an SVM solution $\svm{\lambda}{\dataSet}^{H_k}$ using $\tilde\ell$) achieves a small risk w.r.t. $\tilde\ell$,
then it also achieves a small risk w.r.t. $\ell$, where the conversion between the two types of risks is done by $\Upsilon$.
In particular, if we can show learning rates w.r.t. $\tilde\ell$, then $\Upsilon$ allows us to derive learning rates for the original loss $\ell$.

For example, in case of the hinge loss we have for all measurable $f:\inputSet\rightarrow[-1,1]$ that
\begin{equation*}
    \Risk{\ellClass}{P}(f)-\RiskBayes{\ellClass}{P} \leq \Risk{\ellHinge}{P}(f) - \RiskBayes{\ellHinge}{P},
\end{equation*}
cf. \cite[Theorem~2.31]{SC08}.
This means that at least when restricting to hypotheses taking only values in $[-1,1]$ (which is not really a restriction since $\ellHinge$ can be clipped at 1),
the hinge loss is calibrated w.r.t. classification.
Note that this is a slightly imprecise description, and we refer to \cite[Chapters~3]{SC08} and \cite[Chapter~4]{bach2024learning} for precise statements.

Finally, we connect this to assumptions used to derive learning rates.
In least-squares regression, one uses the \defm{squared loss} $\ell_2: \R\times\R\rightarrow\Rnn$, $\ell_2(y,y')=(y-y')^2$,
and in the context of SVMs (in the sense of regularized empirical risk minimization over RKHSs), this leads to kernel ridge regression (KRR).
Now, the squared loss is classification calibrated (cf. \cite[Section~3.4]{SC08} and \cite[Section~4.1]{bach2024learning} for details),
and KRR is theoretically well-understood, also in the case of distributional inputs in the two-stage sampling setup, cf. \cite{szabo2016learning,fang2020optimal},
including learning rates.
Altogether, this means that we can use KRR in the two-stage sampling setup for classification, and we get even learning rates.
However, we argue that this does \emph{not} mean that a dedicated theory for distributional classification is not necessary.
Let us discuss this in the general setting of a loss function $\ell$ and a surrogate loss $\tilde\ell$, which is calibrated w.r.t. $\ell$.
As is well-known, to get learning rates we need distributional assumptions due to the No-Free-Lunch-Theorem, cf. \cite[Chapter~6]{SC08}.
To derive learning rates for the learning method using the surrogate loss $\tilde\ell$, we need assumptions tailored to the latter.
However, the overall goal is to solve the learning problem w.r.t. to original loss $\ell$, and \emph{assumptions for the surrogate problem might not be appropriate for the original problem}.
For regression problems with the square loss $\ell_2$, one usually invokes smoothness assumptions (for the conditional expectation), and in this context this is a very natural type of assumption.
However, in general this might not be appropriate for classification problems, since the inherent difficult of a classification problem is in general not related to the smoothness of the conditional expectation (here, the class probabilities), but rather the geometry of the classification problem (e.g., whether two classes are appropriately separated).
In the context of classification, assumptions like margin or noise exponents are more appropriate instead of smoothness conditions, cf. \cite[Chapter~8]{SC08} for a discussion.
These considerations suggest to tailor a statistical analysis to the actual learning problem at hand, in the present situation mainly classification, and then invoke assumptions appropriate for the specific learning problem.
In this work, we first use a generic distributional assumption in Section \ref{sec:distrslt:consistencyAndLearningRates}, cf. Assumption \ref{assumption:distrslt:decayApproxErrorFunc}, and then we replace it for the hinge loss and Gaussian kernels with an assumption tailored to the classification setting, cf. Assumption \ref{assumption:distrslt:geometricNoiseAssumptionGaussian}.
In particular, this latter assumption does not impose any explicit smoothness, and allows an intuitive interpretation in the context of classification, cf. Section \ref{sec:distrslt:discussionOfgeometricNoiseAssumptionGaussian}.

\subsection{Background on comparison functions} \label{sec:distrslt:backgroundComparisonFunctions}
For the reader's convenience, we recall some facts about the comparison function formalism as used in control theory, and we refer to \cite{kellett2014compendium} for a thorough overview of this topic.

The main idea of comparison functions is to use classes of functions that model certain behaviours of bounds.
Among the most important of these are class $\cK$ functions, defined as
\begin{equation*}
    \cK=\{f:\Rnn\rightarrow\Rnn \mid f \text{ continuous, strictly increasing}, f(0)=0 \},
\end{equation*}
and class $\cKi$, defined as
\begin{equation*}
    \cKi=\{ f \in \cK \mid \lim_{s\rightarrow\infty} f(s) = \infty \},
\end{equation*}
modelling a generic upper bound or lower bound that depends on some scalar quantity, typically a Lipschitz constant or norm.
Similarly, class $\cL$ functions are defined as
\begin{equation*}
    \cL = \{ f: \Rnn \rightarrow \Rp \mid f \text{ continuous, strictly decreasing,} \lim_{s\rightarrow\infty} f(s)=0 \},
\end{equation*}
which abstracts a bound like $C \exp(-s t)$, $C,s\in\Rp$, with $t$ growing (for example, time).
Finally, class $\cKL$ functions are defined as
\begin{equation*}
    \cKL = \{ \beta: \Rnn\times\Rnn \rightarrow \Rp \mid \beta(\cdot,s) \in \cL \, \forall s\in\Rp,\: \beta(t,\cdot)\in\cKi \forall t \in \Rnn \},
\end{equation*}
and they abstract an upper bound like $(t,s) \mapsto Cs\exp(-ct)$, as appearing in the theory of dynamical systems.
In the present work, a bound as in Assumption \ref{assumption:distrslt:embeddingsEstimationBounds} is typically also of this form.

Operations and relations on classes of comparison functions are defined pointwise, and these function classes enjoy a rich structure and calculus.
For example, for $\alpha_1,\alpha_2\in\cK$ and $\lambda\in\Rp$, $\alpha_1 + \lambda \alpha_2 \in \cK$,
and $\alpha_1 \leq \alpha_2$ means $\alpha_1(s) \leq \alpha_2(s)$ for all $s\in\Rnn$.
For an overview of the calculus with comparison functions, we again refer to \cite{kellett2014compendium} and the many references therein.

Finally, we follow \cite{fiedler2024statistical} and call $(\alpha_i)_{i\in I}$, $I\subseteq\R$, a \defm{nondecreasing family} if $\alpha_i \in \cK$ (or $\alpha_i\in\cKi$) and $\alpha_s \leq \alpha_t$ for $s\leq t$, $s,t\in I$.

Comparison functions are very common in modern nonlinear control, but are rarely used in the theory of machine learning so far.
One main advantage of this formalism is the possibility to quickly derive bounds of a simple, abstract form, and by replacing the various comparison functions with concrete bounds, one can recover specific, precise bounds.
For example, Assumption \ref{assumption:distrslt:kernelForSVM} specifies an abstract, quantitative form of continuity of the canonical feature map, and in Proposition \ref{prop:distrslt:consistencySVMwithKMEs} we specialize this to Hölder continuity to derive concrete conditions for consistency.
Similarly, in Theorem \ref{thm:distrslt:learningRateSVMwithKMEs} we derive learning rates using concrete bounds (leading to a concrete learning rate), whereas in Theorem \ref{thm:distrslt:learningRateSVMgenericApproxErrorFunc} we use abstract bounds described by comparison functions.
\section{Oracle inequalities}
The overall goal of this section is to prove Theorem \ref{thm:distrslt:oracleInequSVM}, which is a two-stage sampling variant of \cite[Theorem~7.22]{SC08}.
Analogous to the proof of this latter result, we first state and prove an oracle inequality for a more general class of learning methods, $\epsilon$-approximate clipped regularized empirical risk minimization.
In Section \ref{sec:distrslt:oraceInequalityCRERM}, we introduce this class of learning methods, and then state and prove the corresponding oracle inequality.
In Section \ref{sec:distrslt:proofOforacleInequSVM}, we then use this to prove Theorem \ref{thm:distrslt:oracleInequSVM}.
\subsection{Oracle inequality for approximate clipped RERMs} \label{sec:distrslt:oraceInequalityCRERM}
We start by introducing $\epsilon$-approximate clipped-risk empirical risk minimization, following the setup from \cite[Section~7.4]{SC08} and adapt it to the present situation of two-stage sampling.

Consider the setting introduced in Section \ref{sec:distrslt:setup}.
In addition, let $(\mathcal{F},d_\mathcal{F})$ be a Polish space (complete and separable metric space) with $\mathcal{F}\subseteq\MeasurableFunctions(\inputSet)$, and assume that the metric $d_\mathcal{F}$ dominates pointwise convergence.
Furthermore, consider a continuous function $\regularizer: \mathcal{F}\rightarrow\Rnn$, which will be the regularizer.

Following \cite[Definition~7.18]{SC08}, for $\epsilon\in\Rnn$, a \defm{$\epsilon$-approximate clipped regularized empirical risk minimization ($\epsilon$-CR-ERM)} w.r.t. $\ell$, $\mathcal{F}$, $\regularizer$ is a map
\begin{equation*}
    \bigcup_{N\in\Np} (\inputSet\times\outputSet)^N \rightarrow \mathcal{F},
    \dataSet \mapsto f_\dataSet
\end{equation*}
such that for all $N\in\Np$ and $\dataSet \in (\inputSet\times\outputSet)^N$ we have
\begin{equation}
    \Risk{\ell}{\dataSet}(\clipped{f_\dataSet}) + \regularizer(f_{\dataSet}) 
    \leq
    \inf_{f \in \mathcal{F}} \Risk{\ell}{\dataSet}(f) + \regularizer(f) + \epsilon.
\end{equation}
In the following, we consider a $\epsilon$-CR-ERM that is measurable, cf. \cite[Lemma~7.19]{SC08}.

Before stating and proving an oracle inequality for the $\epsilon$-CR-ERM in the two-stage sampling setup, we need to introduce a few objects and quantities that will be important in the analysis.
Define
\begin{equation}
    r^\ast = \inf_{f \in \mathcal{F}} \Risk{\ell}{P_\hilbertianEmbed} + \regularizer(f) - \RiskBayes{\ell}{P_\hilbertianEmbed},
\end{equation}
and for $r>r^\ast$ define in addition
\begin{align}
    \mathcal{F}_r = \{ f \in \mathcal{F} \mid \Risk{\ell}{P_\hilbertianEmbed} + \regularizer(f) - \RiskBayes{\ell}{P_\hilbertianEmbed} \leq r \}
\end{align}
and 
\begin{equation}
    \mathcal{H}_r = \{ \ell \pipe \clipped{f} - \ell \pipe f_{\ell,P_\hilbertianEmbed}^\ast \mid f \in \mathcal{F}_r\}.
\end{equation}
Recall that the \defm{empirical Rademacher average} of $H\subseteq\{g: \inputSet\times\outputSet \rightarrow \R \mid f \text{ measurable}\}$
over a data set $D=((x_n,y_n))_n \in(\inputSet\times\outputSet)^N$ is defined as
\begin{equation}
    \Rademacher_D(\mathcal{H}_r,N) = \E\left[\sup_{g \in H} \left| \frac1N \sum_{n=1}^N \sigma_n g(x_n,y_n) \right|\right],
\end{equation}
where $\sigma_1,\ldots,\sigma_N$ are independent Rademacher random variables.
Later on, we need a bound on this measure of complexity, which is described in the following assumption.
\begin{assumption} \label{assumption:distrslt:rademacherBoundForCRERM}
For all $N\in\Np$, there exists $\varphi_N: \Rnn \rightarrow\Rnn$ such that
\begin{align*}
    \varphi_N(4r) & \leq 2\varphi_N(r) \\
    \E_{D \distr P_\hilbertianEmbed^{\otimes N}}\left[\Rademacher_D(\mathcal{H}_r,N)\right] & \leq \varphi_N(r)
\end{align*}   
\end{assumption}
The central ingredient for the foundational oracle inequality will be the following assumption.
\begin{assumption} \label{assumption:distrslt:supremumAndVarianceBound}
There exists $B\in\Rnn$ such that
\begin{equation} \label{eq:distrslt:supremumBound}
    \ell(y,t) \leq B \quad \forall y \in \outputSet, t \in [-M,M].
\end{equation}
Furthermore, there exists a measurable function $f_{\ell,P_\hilbertianEmbed}^\ast: \inputSet\rightarrow [-M,M]$ that achieves the Bayes risk,
i.e., $\Risk{\ell}{P_\hilbertianEmbed}(f_{\ell,P_\hilbertianEmbed}^\ast)=\RiskBayes{\ell}{P_\Pi}$,
there are $V\in\Rnn$, $\vartheta\in[0,1]$ with $V \geq B^{2-\vartheta}$ and such that
\begin{equation} \label{eq:distrslt:varianceBound}
    \E_{P_\hilbertianEmbed}\left[\left(\ell \pipe \clipped{f} - \ell \pipe f_{\ell,P_\hilbertianEmbed}^\ast \right)^2 \right]
    \leq
    V \left(\E_{P_\hilbertianEmbed}\left[\ell \pipe \clipped{f} - \ell \pipe f_{\ell,P_\hilbertianEmbed}^\ast \right]\right)^\vartheta
    \quad
    \forall f \in \mathcal{F}.
\end{equation}
\end{assumption}
\eqref{eq:distrslt:supremumBound} is called a \defm{supremum bound}, and \eqref{eq:distrslt:varianceBound} a \defm{variance bound}.
Finally, we consider loss functions fulfilling the following continuity property.
\begin{assumption} \label{assumption:distrslt:continuityOfLoss}
    There exists a non-decreasing family $(\gamma_{3,T})_{T\in\Rp}$ such that
    \begin{equation}
        |\ell(y,t) - \ell(y,t')| \leq \gamma_{3,T}(|t-t'|)
    \end{equation}
    for all $T\in\Rp$, $y\in \R$, and $|t|,|t'|\leq T$.
\end{assumption}
We are now ready to state the oracle inequality for the $\epsilon$-CR-ERM.
\begin{theorem} \label{thm:distrslt:oracleInequCREM}
Let the loss function $\ell$ fulfill Assumption \ref{assumption:distrslt:continuityOfLoss},
and let Assumptions \ref{assumption:distrslt:rademacherBoundForCRERM} and \ref{assumption:distrslt:supremumAndVarianceBound} hold.
In addition, let $f_0 \in \mathcal{F}$ such that there exist $B_0\in\Rnn$, $\alpha_0\in\cK$, and $B_{f_0}\in\Rnn$ such that
\begin{align}
    \| \ell \pipe f_0 \|_\infty & \leq B_0 \label{eq:distrslt:oracleInequCRERM:f0:1} \\
    |f_0(x)-f_0(x')| & \leq \alpha_0(\|x-x'\|_\hilbertianHS) \quad \forall x,x'\in\inputSet  \label{eq:distrslt:oracleInequCRERM:f0:2} \\
    |f_0(x)| & \leq B_{f_0} \quad \forall x \in \inputSet  \label{eq:distrslt:oracleInequCRERM:f0:3}
\end{align}
holds.
Finally, assume that there exists $\alpha_{\LearningAlg{}}\in\cK$ such that for all $N\in\Np$ and $\dataSet\in (\inputSet\times\outputSet)^N$ we have
\begin{equation} \label{eq:distrslt:oracleInequCRERM:learnMethodContinuity}
    |\clipped{f}_\dataSet(x)-\clipped{f}_\dataSet(x')| \leq \alpha_{\LearningAlg{}}(\|x-x'\|_\hilbertianHS)
    \quad \forall x,x'\in\inputSet.
\end{equation}
Then for all $\tau\geq 1$ and all $r\in\Rp$ with
\begin{equation}
    r \geq \max\left\{
        30\varphi_N(r), 
        \left(\frac{72 V \tau}{N}\right)^{\frac{1}{2-\vartheta}}, 
        \frac{5 B_0\tau}{N}, 
        r^\ast 
    \right\}
\end{equation}
it holds with probability at least $1-4e^{-\tau}$ that
\begin{align} \label{eq:distrslt:oracleInequCRERM}
    \Risk{\ell}{P_\hilbertianEmbed}(\clipped{f}_{\dataSet_{\hat\hilbertianEmbed}}) 
        + \regularizer(f_{\dataSet_{\hat\hilbertianEmbed}}) - \RiskBayes{\ell}{P_\hilbertianEmbed}
    & \leq 6\left(\Risk{\ell}{P_\hilbertianEmbed}(f_0) + \regularizer(f_0) - \RiskBayes{\ell}{P_\hilbertianEmbed} \right)
        + 3\epsilon + 3r \nonumber \\
    & + \frac{3}{N} \sum_{n=1}^N (\gamma_{3,B_{f_0}} \circ \alpha_0 + \gamma_{3,M} \circ \alpha_{\LearningAlg{}})(B_n(e^{-\tau}/N)). 
\end{align}
\end{theorem}
\begin{remark}
Assumption \ref{assumption:distrslt:continuityOfLoss} implies the supremum bound in Assumption \ref{assumption:distrslt:supremumAndVarianceBound}.
However, for emphasis, and to allow the use of a potentially less conservative bound, we kept the supremum bound as a separate assumption.
Similarly, Assumption \ref{assumption:distrslt:continuityOfLoss} together with \eqref{eq:distrslt:oracleInequCRERM:f0:3} implies \eqref{eq:distrslt:oracleInequCRERM:f0:1}, but the bound $B_0$ might be tighter.
Finally, \eqref{eq:distrslt:oracleInequCRERM:f0:2} implies a bound like \eqref{eq:distrslt:oracleInequCRERM:f0:3}, but again $B_{f_0}$ might be less conservative.
\end{remark}
We now turn to the proof of Theorem \ref{thm:distrslt:oracleInequCREM}.
The high-level strategy is to use continuity properties to go from the accessible data set $\dataSet_{\hat\hilbertianEmbed}$ to the inaccessible first-stage sampling data set $\dataSetBar_\hilbertianEmbed$, and then apply standard results there.
This strategy has been introduced already by \cite{szabo2015two,lopezpaz2015towards} and used also, for example, by \cite{fiedler2024statistical} (also for generic Hilbertian embeddings).
In our case, we use the approach to apply the oracle inequality \cite[Theorem~7.20]{SC08}, and while the high-level strategy is standard, some work is needed to be able to use this result in the present setting.
\begin{proof}
For $f \in \mathcal{F}$ define
\begin{equation*}
    h_f(x,y) = \ell \pipe f - \ell \pipe f_{\ell,P_\hilbertianEmbed}^\ast
\end{equation*}
and note that
\begin{equation*}
    \Risk{\ell}{P_\hilbertianEmbed}(f) - \RiskBayes{\ell}{P_\hilbertianEmbed}
    =
    \E_{P_\hilbertianEmbed}[h_f]
\end{equation*}
and
\begin{equation*}
    \Risk{\ell}{\bar\dataSet_\hilbertianEmbed}(f) - \Risk{\ell}{\bar\dataSet_\hilbertianEmbed}(f_{\ell,P_\hilbertianEmbed}^\ast)
    =
    \E_{\bar\dataSet_\hilbertianEmbed}[h_f].
\end{equation*}
We have
\begin{align*}
    & \Risk{\ell}{P_\hilbertianEmbed}(\clipped{f}_{\dataSet_{\hat\hilbertianEmbed}}) 
        + \regularizer(f_{\dataSet_{\hat\hilbertianEmbed}}) - \RiskBayes{\ell}{P_\hilbertianEmbed} 
    =  \Risk{\ell}{\dataSet_{\hat\hilbertianEmbed}}(\clipped{f}_{\dataSet_{\hat\hilbertianEmbed}}) 
        + \regularizer(f_{\dataSet_{\hat\hilbertianEmbed}}) 
        - \left( 
            \Risk{\ell}{\dataSet_{\hat\hilbertianEmbed}}(f_0) 
            + \regularizer(f_0)
        \right) \\
        & \hspace{1cm} + \Risk{\ell}{\dataSet_{\hat\hilbertianEmbed}}(f_0) + \regularizer(f_0)
            +  \Risk{\ell}{P_\hilbertianEmbed}(\clipped{f}_{\dataSet_{\hat\hilbertianEmbed}}) - \Risk{\ell}{\dataSet_\hilbertianEmbed}(\clipped{f}_{\dataSet_{\hat\hilbertianEmbed}}) 
        - \RiskBayes{\ell}{P_\hilbertianEmbed} \\
    & \hspace{0.5cm} \leq \Risk{\ell}{\dataSet_{\hat\hilbertianEmbed}}(f_0) + \regularizer(f_0)
        + \Risk{\ell}{P_\hilbertianEmbed}(\clipped{f}_{\dataSet_{\hat\hilbertianEmbed}}) 
            - \Risk{\ell}{\dataSet_\hilbertianEmbed}(\clipped{f}_{\dataSet_{\hat\hilbertianEmbed}}) 
        - \RiskBayes{\ell}{P_\hilbertianEmbed}
        + \epsilon \\
    & \hspace{0.5cm} = 
        \underbrace{\Risk{\ell}{\dataSet_{\hat\hilbertianEmbed}}(f_0) 
            -
        \Risk{\ell}{\bar\dataSet_{\hilbertianEmbed}}(f_0)}_{=I}
        + %
        \underbrace{\Risk{\ell}{\bar\dataSet_{\hilbertianEmbed}}(\clipped{f}_{\dataSet_{\hat\hilbertianEmbed}}) 
            - 
        \Risk{\ell}{\dataSet_{\hat\hilbertianEmbed}}(\clipped{f}_{\dataSet_{\hat\hilbertianEmbed}})}_{=II}
        + %
        \Risk{\ell}{P_\hilbertianEmbed}(f_0) + \regularizer(f_0)   - \RiskBayes{\ell}{P_\hilbertianEmbed} \\
        & \hspace{1cm} + %
            \Risk{\ell}{\bar\dataSet_{\hilbertianEmbed}}(f_0) 
                -
            \Risk{\ell}{P_\hilbertianEmbed}(f_0)
            + %
            \Risk{\ell}{P_\hilbertianEmbed}(\clipped{f}_{\dataSet_{\hat\hilbertianEmbed}})
                -
            \Risk{\ell}{\bar\dataSet_{\hilbertianEmbed}}(\clipped{f}_{\dataSet_{\hat\hilbertianEmbed}}) \\
    & \hspace{0.5cm} = I + II +  \Risk{\ell}{P_\hilbertianEmbed}(f_0) + \regularizer(f_0)   - \RiskBayes{\ell}{P_\hilbertianEmbed}
        + \underbrace{\E_{\bar\dataSet_\hilbertianEmbed}[h_{f_0}] - \E_{P_\hilbertianEmbed}[h_{f_0}]}_{III}
        + \underbrace{\E_{P_\hilbertianEmbed}[h_{\clipped{f}_{\dataSet_{\hat\hilbertianEmbed}}}] 
            - 
            \E_{\bar\dataSet_\hilbertianEmbed}[h_{\clipped{f}_{\dataSet_{\hat\hilbertianEmbed}}}]}_{IV}
        + \epsilon,
\end{align*}
where we used the definition of an $\epsilon$-CR-ERM in the inequality.

Observe now that we can use the proof of \cite[Theorem~7.20]{SC08} to find that with probability at least $1-2e^{-\tau}$ it holds that
\begin{equation*}
    III = \E_{\bar\dataSet_\hilbertianEmbed}[h_{f_0}] - \E_{P_\hilbertianEmbed}[h_{f_0}]
    \leq \E_{P_\hilbertianEmbed}[h_{f_0}] + \left(\frac{2V\tau}{N}\right)^{\frac{1}{2-\vartheta}}
        + \frac{4B\tau}{3N} + \frac{7 B_0\tau}{6N},
\end{equation*}
cf. \cite[(7.42)]{SC08}.
Similarly, cf. the arguments before (7.44) in the proof of \cite[Theorem~7.20]{SC08}, with probability at least $1-e^{-\tau}$ also {\small
\begin{align*}
    IV= \E_{P_\hilbertianEmbed}\left[h_{\clipped{f}_{\dataSet_{\hat\hilbertianEmbed}}}\right] - \E_{\bar\dataSet_\hilbertianEmbed}\left[h_{\clipped{f}_{\dataSet_{\hat\hilbertianEmbed}}}\right]
    & \leq \left( 
        \E_{P_\hilbertianEmbed}\left[h_{\clipped{f}_{\dataSet_{\hat\hilbertianEmbed}}}\right] 
        + \regularizer(f_{\dataSet_{\hat\hilbertianEmbed}})
        \right)
        \left( \frac{10 \varphi_N (r)}{r} + \sqrt{\frac{2V \tau}{n r^{2 - \vartheta}}} + \frac{28 B \tau}{3Nr} \right) \\
    & \hspace{0.5cm} + 10 \varphi_N (r) + \sqrt{\frac{2V \tau r^{\vartheta}}{N}} + \frac{28 B \tau}{3N}
\end{align*} }
holds.

Let us turn to the two remaining terms. With probability at least $1-e^{-\tau}$ we have
\begin{align*}
    \Risk{\ell}{\dataSet_{\hat\hilbertianEmbed}}(f_0) - \Risk{\ell}{\bar\dataSet_{\hilbertianEmbed}}(f_0)
    & = \frac1N \sum_{n=1}^N 
        \ell(y_n, f_0(\hat\hilbertianEmbed S^{(n)})) 
        - 
        \ell(y_n, f_0(\hilbertianEmbed Q_n)) \\
    & \leq  \frac1N \sum_{n=1}^N 
        |\ell(y_n, f_0(\hat\hilbertianEmbed S^{(n)})) 
        - 
        \ell(y_n, f_0(\hilbertianEmbed Q_n))| \\
    & \leq \frac1N \sum_{n=1}^N 
        \gamma_{3,B_{f_0}}(|f_0(\hat\hilbertianEmbed S^{(n)}) -  f_0(\hilbertianEmbed Q_n)|) \\
    & \leq \frac1N \sum_{n=1}^N 
        \gamma_{3,B_{f_0}} \circ \alpha_0(\|\hat\hilbertianEmbed S^{(n)} - \hilbertianEmbed Q_n\|_\hilbertianHS) \\
    & \leq \frac1N \sum_{n=1}^N 
        \left(\gamma_{3,B_{f_0}} \circ \alpha_0\right)\left(\hilbertianEmbedBound(M^{(n)}, e^{-\tau}/N)\right)
\end{align*}
where we used Assumption \ref{assumption:distrslt:continuityOfLoss} together with  \eqref{eq:distrslt:oracleInequCRERM:f0:3} for the second inequality,
\eqref{eq:distrslt:oracleInequCRERM:f0:2} in the third inequality,
and finally the embedding estimation bound together with the union bound in the last inequality.

Similarly, with probability at least $1-e^{-\tau}$ (on the same event as before) we also have
\begin{align*}
    II = \Risk{\ell}{\bar\dataSet_{\hilbertianEmbed}}(\clipped{f}_{\dataSet_{\hat\hilbertianEmbed}}) 
            - 
    \Risk{\ell}{\dataSet_{\hat\hilbertianEmbed}}(\clipped{f}_{\dataSet_{\hat\hilbertianEmbed}})
    & =
    \frac1N \sum_{n=1}^N 
        \ell(y_n, \clipped{f}_{\dataSet_{\hat\hilbertianEmbed}}(\hat\hilbertianEmbed S^{(n)})) 
        - 
        \ell(y_n, \clipped{f}_{\dataSet_{\hat\hilbertianEmbed}}(\hilbertianEmbed Q_n)) \\
    & \leq \frac1N \sum_{n=1}^N 
        |\ell(y_n, \clipped{f}_{\dataSet_{\hat\hilbertianEmbed}}(\hat\hilbertianEmbed S^{(n)})) 
        - 
        \ell(y_n, \clipped{f}_{\dataSet_{\hat\hilbertianEmbed}}(\hilbertianEmbed Q_n))| \\
    & \leq  \frac1N \sum_{n=1}^N 
        \gamma_{3,M}(| \clipped{f}_{\dataSet_{\hat\hilbertianEmbed}}(\hat\hilbertianEmbed S^{(n)}) -   \clipped{f}_{\dataSet_{\hat\hilbertianEmbed}}(\hilbertianEmbed Q_n)|) \\
    & \leq \frac1N \sum_{n=1}^N 
        \gamma_{3,M} \circ \alpha_{\LearningAlg{}}(\|\hat\hilbertianEmbed S^{(n)} - \hilbertianEmbed Q_n\|_\hilbertianHS) \\
    & \leq \frac1N \sum_{n=1}^N 
    \left(\gamma_{3,M} \circ \alpha_{\LearningAlg{}}\right)\left(\hilbertianEmbedBound(M^{(n)}, e^{-\tau}/N)\right),
\end{align*}
where we used the definition of clipping in the third inequality.
Altogether, with probability at least $1-e^{-\tau}$ we have
\begin{equation*}
    I + II \leq \frac1N \sum_{n=1}^N 
    \left(\gamma_{3,B_{f_0}} \circ \alpha_0 + \gamma_{3,M} \circ \alpha_{\LearningAlg{}}\right)(\hilbertianEmbedBound(M^{(n)}, e^{-\tau}/N)).
\end{equation*}

Define now for notational simplicity
\begin{align*}
    \mathcal{E} & = \Risk{\ell}{P_\hilbertianEmbed}(\clipped{f}_{\dataSet_{\hat\hilbertianEmbed}}) 
    + \regularizer(f_{\dataSet_{\hat\hilbertianEmbed}}) - \RiskBayes{\ell}{P_\hilbertianEmbed}  \\
    \Delta_N & = \frac1N \sum_{n=1}^N 
    \left(\gamma_{3,B_{f_0}} \circ \alpha_0 + \gamma_{3,M} \circ \alpha_{\LearningAlg{}}\right)(\hilbertianEmbedBound(M^{(n)}, e^{-\tau}/N)).
\end{align*}
Combining our estimates so far, we have with probability at least $1-4e^{-\tau}$ that
\begin{align*}
    \mathcal{E} & \leq \E_{P_\hilbertianEmbed}[h_{f_0}] + \regularizer(f_0) + \left(\frac{2V\tau}{N}\right)^{\frac{1}{2-\vartheta}}
    + \frac{4B\tau}{3N} + \frac{7 B_0\tau}{6N} 
    + \mathcal{E}
        \left( \frac{10 \varphi_N (r)}{r} + \sqrt{\frac{2V \tau}{n r^{2 - \vartheta}}} + \frac{28 B \tau}{3Nr} \right) \\
    & \hspace{0.5cm} + 10 \varphi_N (r) + \sqrt{\frac{2V \tau r^{\vartheta}}{N}} + \frac{28 B \tau}{3N} + \Delta_N + \epsilon
\end{align*}
Using elementary estimates, cf. the proof of \cite[Theorem~7.20]{SC08}, this implies that
\begin{align*}
    \mathcal{E} & \leq 2\E_{P_\hilbertianEmbed}[h_{f_0}] + \regularizer(f_0)  + \left(\frac{2V\tau}{N}\right)^{\frac{1}{2-\vartheta}}
    + \frac{7 B_0\tau}{6N} 
    + \frac{17}{27}\mathcal{E} +\frac{22}{27}r + \sqrt{\frac{2V \tau r^{\vartheta}}{N}} + \frac{28 B \tau}{3N} + \Delta_N + \epsilon
\end{align*}
Rearranging and some elementary estimates lead to
\begin{align*}
    \mathcal{E} & \leq  6\E_{P_\hilbertianEmbed}[h_{f_0}] + 3\regularizer(f_0) + \frac{22}{9}r + 3\Delta_N + 3\epsilon,
\end{align*}
which establishes the result.
\end{proof}
\subsection{Proof of Theorem \ref{thm:distrslt:oracleInequSVM} (Oracle inequality for SVMs with generic Hilbertian embeddings)} \label{sec:distrslt:proofOforacleInequSVM}
We follow the proof strategy of \cite[Theorem~7.22]{SC08} and deduce the result from Theorem \ref{thm:distrslt:oracleInequCREM} with $\epsilon=0$, $\mathcal{F}=H_k$, and $\regularizer = \lambda \|\cdot\|_k^2$.

Exactly as in the proof of \cite[Theorem~7.22]{SC08}, we can derive
\begin{equation*}
    \varphi_N(r) = \tilde C \LocLip{\ell}{M}\|k\|_\infty \sqrt{\frac{\ln(N)r}{N\lambda}} + \sqrt{\frac{\ln(16)}{N}}B
\end{equation*}
as a suitable bound on the empirical Rademacher averages, cf. Assumption \ref{assumption:distrslt:rademacherBoundForCRERM}, where $\tilde C$ is a universal constant.

Furthermore, the supremum bound in Assumption \ref{assumption:distrslt:supremumAndVarianceBound}(ii) is fulfilled with $V=B^2$ and $\vartheta=0$.

For $f_0$ in Theorem \ref{thm:distrslt:oracleInequCREM}, we set $f_0=f_{P_\hilbertianEmbed,\lambda}$.
Since
\begin{equation*}
    \lambda \|f_{P_\hilbertianEmbed,\lambda}\|_k^2 
    \leq
    \Risk{\ell}{P_\hilbertianEmbed}(f_{P_\hilbertianEmbed,\lambda}) +  \lambda \|f_{P_\hilbertianEmbed,\lambda}\|_k^2 - \RiskOpt{\ell}{P_\hilbertianEmbed}{H_k}
    = \Aef(\lambda),
\end{equation*}
we have $ \|f_{P_\hilbertianEmbed,\lambda}\|_k \leq \sqrt{\Aef(\lambda)/\lambda}$, and hence
\begin{equation*}
    \| f_{P_\hilbertianEmbed,\lambda}\|_\infty \leq \|k\|_\infty  \|f_{P_\hilbertianEmbed,\lambda}\|_k
    \leq \|k\|_\infty \sqrt{\frac{\Aef(\lambda)}{\lambda}}.
\end{equation*}
We use this to get
\begin{align*}
    \ell(y,f_{P_\hilbertianEmbed,\lambda}(x)) & \leq \ell(y,0) + |\ell(y,f_{P_\hilbertianEmbed,\lambda}(x)) - \ell(x,y,0)| \\
    & \leq B + \LocLip{\ell}{\|f_{P_\hilbertianEmbed,\lambda}\|_\infty} \|f_{P_\hilbertianEmbed,\lambda}\|_\infty \\
    & \leq B + \LocLip{\ell}{\|k\|_\infty\sqrt{\Aef(\lambda)/\lambda}}\|k\|_\infty\sqrt{\Aef(\lambda)/\lambda} = B_0,
\end{align*}
so we get $\|\ell \pipe f_{P_\hilbertianEmbed,\lambda}\|_\infty \leq B_0$.

Next, we have
\begin{align*}
    r^\ast & = \inf_{f \in H_k} \Risk{\ell}{P_\hilbertianEmbed}(\clipped{f}) + \lambda \|f\|_k^2 - \RiskBayes{\ell}{P_\hilbertianEmbed} \\
    & \leq \inf_{f \in H_k} \Risk{\ell}{P_\hilbertianEmbed}(f) + \lambda \|f\|_k^2 - \RiskOpt{\ell}{P_\hilbertianEmbed}{H_k}
        +  \RiskOpt{\ell}{P_\hilbertianEmbed}{H_k} - \RiskBayes{\ell}{P_\hilbertianEmbed} \\
    & = \Aef(\lambda) + \RiskOpt{\ell}{P_\hilbertianEmbed}{H_k} - \RiskBayes{\ell}{P_\hilbertianEmbed},
\end{align*}
where we used the clippability of $\ell$ in the inequality.

We can now set
\begin{align*}
    r & = 60^2 \tilde C^2 \LocLip{\ell}{M}^2 \|k\|_\infty^2 \frac{\ln(N)}{N\lambda} 
        + 100 \frac{B\tau}{\sqrt{N}} \\
        & \hspace{0.5cm} + \frac{5\tau}{N}\LocLip{\ell}{\|k\|_\infty\sqrt{\Aef(\lambda)/\lambda}}\|k\|_\infty\sqrt{\Aef(\lambda)/\lambda}
        + \Aef(\lambda) + \RiskOpt{\ell}{P_\hilbertianEmbed}{H_k} - \RiskBayes{\ell}{P_\hilbertianEmbed},
\end{align*}
which ensures that all requirements on $r$ in Theorem \ref{thm:distrslt:oracleInequCREM} are fulfilled.

Finally, Assumption \ref{assumption:distrslt:continuityOfLoss} is fulfilled with $\gamma_{3,T}=\LocLip{\ell}{T}$,
and we can set $\alpha_0 = \sqrt{\Aef(\lambda)/\lambda}\alpha_k$ in \eqref{eq:distrslt:oracleInequCRERM:f0:2},
and $B_{f_0}=\|k\|_\infty\sqrt{\Aef(\lambda)/\lambda}$ in \eqref{eq:distrslt:oracleInequCRERM:f0:3}.
Furthermore, since for all $x,x'\in\inputSet$ we have
\begin{align*}
    |f_{D_{\hat\hilbertianEmbed},\lambda}(x)-f_{D_{\hat\hilbertianEmbed},\lambda}(x')|
    & \leq \|f_{D_{\hat\hilbertianEmbed},\lambda}\|_k \|\fm_k(x)-\fm_k(x')\|_k \\
    & \leq \|f_{D_{\hat\hilbertianEmbed},\lambda}\|_k \alpha_k(\|x-x'\|_\hilbertianHS),
\end{align*}
and
\begin{equation*}
    \lambda \|f_{D_{\hat\hilbertianEmbed},\lambda}\|_k^2 
    \leq \Risk{\ell}{D_{\hat\hilbertianEmbed}}(f_{D_{\hat\hilbertianEmbed},\lambda}) + \lambda \|f_{D_{\hat\hilbertianEmbed},\lambda}\|_k^2 
    \leq \Risk{\ell}{D_{\hat\hilbertianEmbed}}(0) \leq B
\end{equation*}
implies $\|f_{D_{\hat\hilbertianEmbed},\lambda}\|_k \leq \sqrt{B/\lambda}$,
we can use $\alpha_{\LearningAlg{}}=\sqrt{B/\lambda}\alpha_k$ in \eqref{eq:distrslt:oracleInequCRERM:learnMethodContinuity}.
Altogether, the last term in \eqref{eq:distrslt:oracleInequCRERM} becomes
\begin{align*}
    \frac{3}{N} \sum_{n=1}^N (\LocLip{\ell}{\|k\|_\infty\sqrt{\Aef(\lambda)/\lambda}}\alpha_k + \LocLip{\ell}{M} \sqrt{B/\lambda}\alpha_k)(\hilbertianEmbedBound(M^{(n)},e^{-\tau}/N)).
\end{align*}The result now follows from Theorem \ref{thm:distrslt:oracleInequCREM}.

\section{Consistency and Learning Rates}
\subsection{Proof of Proposition \ref{prop:distrslt:consistencySVMwithKMEs} (Consistency of SVMs with KMEs)} \label{sec:distrslt:proofOfonsistencySVMwithKMEs}
Proposition \ref{prop:distrslt:consistencySVMwithKMEs} follows immediately from Proposition \ref{prop:distrslt:consistencySVM},
once we have checked that the second condition in \eqref{eq:distrslt:consistencySVM:conditionsForSequences} holds.

According to Proposition \ref{prop:distrslt:kmeBackground}, we can use
\begin{equation*}
    \hilbertianEmbedBound(M,\delta) = 2\sqrt{\frac{\|\kKME\|_\infty^2}{M}} + \sqrt{\frac{2\|\kKME\|_\infty \ln(1/\delta)}{M}},
\end{equation*}
so the term in the last condition in \eqref{eq:distrslt:consistencySVM:conditionsForSequences} becomes
\begin{equation*}
     \frac{1}{\sqrt{\lambda_N}} \alpha_k(\hilbertianEmbedBound(M_N,1/N)) 
    = 
    \frac{C_k}{\sqrt{\lambda_N}}\left(2\sqrt{\frac{\|\kKME\|_\infty^2}{M_N}} + \sqrt{\frac{2\|\kKME\|_\infty \ln(N)}{M_N}}\right)^\alpha,
\end{equation*}
which is of the form
\begin{equation*}
    \frac{C_1}{\sqrt{\lambda_N}} \left(\sqrt{\frac{1}{M_N}} + \sqrt{\frac{\ln(C_2 N)}{M_N}}\right)^\alpha
    =
    \frac{C_1}{\sqrt{\lambda_N}} M_N^{-\frac{\alpha}{2}}\left(1 + \sqrt{\ln(C_2 N)}\right)^\alpha
\end{equation*}
for constants $C_1,C_2\in\Rp$.
Furthermore, since
\begin{equation*}
    \left(1 + \sqrt{\ln(C_2 N)}\right) \leq c_\alpha(1 + C_2^{\frac{\alpha}{2}} + \ln(N)^{\frac{\alpha}{2}})
\end{equation*}
for an appropriate constant $c_\alpha\in\Rp$,
we can upper bound the preceding term by
\begin{align*}
      C_1 c_\alpha 
        \frac{1}{\sqrt{\lambda_N}} M_N^{-\frac{\alpha}{2}}\left(1 + C_2^{\frac{\alpha}{2}} + \ln(N)^{\frac{\alpha}{2}}\right) 
     =
     C_1 c_\alpha \left(
        \sqrt{\frac{1}{\lambda_N M_N^\alpha}}
        + 
        \sqrt{\frac{C_2^\alpha}{\lambda_N M_N^\alpha}}
        +
        \sqrt{\frac{\ln(N)^\alpha}{\lambda_N M_N^\alpha}} \right).
\end{align*}
Finally, $\frac{\ln(N)^\alpha}{\lambda_N M_N^\alpha} \rightarrow 0$ implies also $\frac{1}{\lambda_N M_N^\alpha} \rightarrow 0$, so this ensures convergence to zero.
\subsection{Proof of Proposition \ref{thm:distrslt:learningRateSVMwithKMEs} (Learning Rates for SVMs with KMEs)} \label{sec:distrslt:proofOflearningRateSVMwithKMEs}
Combining Theorem \ref{thm:distrslt:oracleInequSVM} with the additional assumptions in the statement, we get that for all $\tau\geq 1$, with probability at least $1-4e^{-\tau}$, $ \Risk{\ell}{P}(\clipped{f}_{\dataSet_{\hat\hilbertianEmbed},\lambda} \circ \hilbertianEmbed) - \RiskBayes{\ell}{P}$
is upper bounded by 
\begin{align*}
    & C_1 \lambda^\beta + C_2 N^{-1}\ln(N) \lambda^{-1} + C_3 \tau N^{-1} \lambda^{\frac{\beta-1}{2}} 
    + C_4\left(\lambda^{\frac{\beta-1}{2}} + \lambda^{-\frac12}\right)M^{-\frac{\alpha}{2}}(1 + \sqrt{\ln(N/e^{-\tau})})^\alpha \\
    & \hspace{0.5cm} + C_5 \tau N^{-\frac12},
\end{align*} 
for appropriate constants $C_1,\ldots,C_5$ (independent of $\lambda$, $N$, and $M$).
For a suitable constant $c_\alpha \in \Rp$ we furthermore have
\begin{align*}
    (1 + \sqrt{\ln(N/e^{-\tau})})^\alpha & \leq c_\alpha + c_\alpha \tau^{\frac{\alpha}{2}} + c_\alpha \ln(N)^{\frac{\alpha}{2}},
\end{align*}
and since $0<\alpha\leq 2$, and $\tau,N\geq 1$, we get with an appropriate constant $\tilde C$ (independent of $\tau,N$) that
\begin{equation*}
    (1 + \sqrt{\ln(N/e^{-\tau})})^\alpha \leq \tilde C(1+\tau+\ln(N)).
\end{equation*}
We can now find a constant $C\in\Rp$ (independent of $\lambda,N,M$) such that with probability at least $1-4e^{-\tau}$, $\Risk{\ell}{P}(\clipped{f}_{\dataSet_{\hat\hilbertianEmbed},\lambda} \circ \hilbertianEmbed) - \RiskBayes{\ell}{P}$ is upper bounded by
\begin{align*}
    C \tau \ln(N)\left(\lambda^\beta + N^{-1} \lambda^{-1} + N^{-1}\lambda^{\frac{\beta-1}{2}} + (\lambda^{\frac{\beta-1}{2}}+\lambda^{-\frac12})M^{-\frac{\alpha}{2}} + N^{-\frac12} \right),
\end{align*}
and by restriction to $0<\lambda\leq 1$, so that we have $\lambda^{-\frac12} \leq \lambda^{-1}$, we can upper bound the preceding term by
\begin{equation*}
    C \tau \ln(N)\left(\lambda^\beta + N^{-1} \lambda^{-1} + N^{-1}\lambda^{\frac{\beta-1}{2}} + (\lambda^{\frac{\beta-1}{2}}+\lambda^{-1})M^{-\frac{\alpha}{2}} + N^{-\frac12} \right).
\end{equation*}
Choosing $M=N^\gamma$ for some $\gamma\in\Rp$ leads to
\begin{align*}
    C \tau \ln(N)\left(\lambda^\beta + N^{-1} \lambda^{-1} + N^{-1}\lambda^{\frac{\beta-1}{2}} + (\lambda^{\frac{\beta-1}{2}}+\lambda^{-1})N^{-\frac{\gamma\alpha}{2}} + N^{-\frac12} \right),
\end{align*}
so by setting $\gamma=2/\alpha$ and adjusting $C$, we can upper bound this by
\begin{align*}
    C \tau \ln(N)\left(\lambda^\beta + N^{-1} \lambda^{-1} + N^{-1}\lambda^{\frac{\beta-1}{2}} + N^{-\frac12} \right).
\end{align*}
Setting $\lambda=N^{-\frac{1}{\beta+1}}$ and adjusting $C$ once again, we get the upper bound
\begin{equation*}
     C \tau \ln(N) \left( N^{-\frac{\beta}{\beta+1}} + N^{-\frac12} \right)
\end{equation*}
and since $N^{-\frac{\beta}{\beta+1}} \geq N^{-\frac12}$, we finally arrive at the upper bound
\begin{equation*}
    C \tau \ln(N) N^{-\frac{\beta}{\beta+1}}.
\end{equation*}
Rescaling $\tau$ then establishes the result.

\subsection{Learning Rate for Generic Hilbertian Embeddings} \label{sec:distrslt:learningRateGenericHilbertianEmbedding}
In this section, we state and prove a result on learning rates for SVMs with a generic Hilbertian embedding, for which we use the comparison function formalism, cf. Section \ref{sec:distrslt:backgroundComparisonFunctions} for some background.
\begin{theorem} \label{thm:distrslt:learningRateSVMgenericApproxErrorFunc}
Consider the situation of Theorem \ref{thm:distrslt:oracleInequSVM}, and let in addition Assumption \ref{assumption:distrslt:decayApproxErrorFunc} hold.
Furthermore, assume that there exist $\gamma_1,\gamma_2\in\cKi$ and $\rho\in\cL$ with
\begin{equation}
\alpha_k\left(\hilbertianEmbedBound(M,e^{-\tau}/N)\right) \leq \rho(M)\gamma_1(1/e^{-\tau})\gamma_2(N).
\end{equation}
If $(M_N)_N$ grows at least as $\rho^{-1}(N^{-1})$, and $(\lambda_N)_N$ decays as $N^{-\frac{1}{\beta+1}}$, then a learning rate of
$\max\{\ln(N),\gamma_2(N)\}N^{-\frac{\beta}{\beta+1}}$ is achieved.
\end{theorem}
\begin{proof}
Combining Theorem \ref{thm:distrslt:oracleInequSVM} with the additional assumptions in the statement, we get that for all $\tau\geq 1$, with probability at least $1-4e^{-\tau}$, $ \Risk{\ell}{P}(\clipped{f}_{\dataSet_{\hat\hilbertianEmbed},\lambda} \circ \hilbertianEmbed) - \RiskBayes{\ell}{P}$
is upper bounded by 
\begin{align*}
    & C_1 \lambda^\beta + C_2 N^{-1}\ln(N) \lambda^{-1} + C_3 \tau N^{-1} \lambda^{\frac{\beta-1}{2}} 
    + C_4\left(\lambda^{\frac{\beta-1}{2}} + \lambda^{-\frac12}\right)\alpha_k\left(\hilbertianEmbedBound(M,e^\tau/N)\right) \\
    & \hspace{0.5cm} + C_5 \tau N^{-\frac12},
\end{align*} 
for appropriate constants $C_1,\ldots,C_5$ (independent of $\lambda$, $N$, and $M$).
Using the additional assumption, we can upper bound this by
\begin{equation*}
    C \tau \ln(N) \left( \lambda^\beta + N^{-1}\lambda^{-1}  + N^{-1} \lambda^{\frac{\beta-1}{2}} + N^{-1} \right)
    + 
    C \left(\lambda^{\frac{\beta-1}{2}} + \lambda^{-\frac12}\right) \rho(M)\gamma_1(1/e^{-\tau})\gamma_2(N)
\end{equation*}
for an absolute constant $C\in\Rp$ (independent of $N,M,\tau,\lambda$), which we can further upper bound by
\begin{equation*}
    C \max\{\tau, \gamma_1(1/e^{-\tau})\} \left(
        \lambda^\beta + N^{-1}\lambda^{-1}  + N^{-1} \lambda^{\frac{\beta-1}{2}} + \left(\lambda^{\frac{\beta-1}{2}} + \lambda^{-\frac12}\right)\rho(M) + N^{-1}
    \right).
\end{equation*}
Restricting to $0<\lambda\leq 1$, so that $\lambda^{-\frac12} \leq \lambda^{-1}$, and adjusting $C$ appropriately, we can in turn upper bound this by
\begin{equation*}
     C \max\{\tau, \gamma_1(1/e^{-\tau})\} \left(
        \lambda^\beta + N^{-1}\lambda^{-1}  + N^{-1} \lambda^{\frac{\beta-1}{2}} + \left(\lambda^{\frac{\beta-1}{2}} + \lambda^{-1}\right)\rho(M) + N^{-1}
    \right).
\end{equation*}
Choosing $M=\rho^{-1}(N^{-1})$ then leads to
\begin{equation*}
         C \max\{\tau, \gamma_1(1/e^{-\tau})\} \left(
        \lambda^\beta + N^{-1}\lambda^{-1}  + N^{-1} \lambda^{\frac{\beta-1}{2}} + N^{-1}
    \right).
\end{equation*}
Setting $\lambda=N^{-\frac{1}{\beta+1}}$ and adjusting $C$ once again, we get the upper bound
\begin{equation*}
     C \tau \ln(N) \left( N^{-\frac{\beta}{\beta+1}} + N^{-\frac12} \right)
\end{equation*}
and since $N^{-\frac{\beta}{\beta+1}} \geq N^{-\frac12}$, we finally arrive at the upper bound
\begin{equation*}
    C \tau \ln(N) N^{-\frac{\beta}{\beta+1}}.
\end{equation*}
Rescaling $\tau$ then establishes the result.
\end{proof}
\begin{remark}
The additional assumption in Theorem \ref{thm:distrslt:learningRateSVMgenericApproxErrorFunc} is rather weak.
In general, $\hilbertianEmbedBound(\cdot, \sigma)$ will be increasing for fixed $\sigma$, and $\hilbertianEmbedBound(\tau,\cdot)$ will be increasing for $\tau$ fixed.
This means that $(t,s)\mapsto \alpha_k(\hilbertianEmbedBound(t,1/s))$ behaves like a $\cKL$ function, so we can upper bound this by such a function $\beta_\hilbertianEmbed$.
In turn, Sontag's $\cKL$-Lemma ensures that there exist $\tilde \gamma_1,\gamma_2\in\cKi$, $\tilde \rho\in\cL$ such that
$\beta_\hilbertianEmbed(t,s)\leq\tilde\gamma_1(\tilde\rho(t)\tilde\gamma_2(s))$ holds.
Under mild assumptions, cf. \cite[Lemma~8]{kellett2014compendium}, there exist $\hat\rho\in\cL$ and $\hat\gamma\in\cK$ such that
$\tilde\gamma_1(\tilde\rho(t)\tilde\gamma_2(s)) \leq \hat\rho(t)\hat\gamma(s)$ holds.
Furthermore, setting $s=r_1r_2$ and observing that $(r_1,r_2) \mapsto \hat\gamma(r_1r_2)$ behaves like a $\cK$ function when either of the arguments is fixed (to a positive value), we can use \cite[Lemma~11]{kellett2014compendium} to get the existence of $\gamma\in\cK$ with $\hat\gamma(r_1r_2)\leq \gamma(r_1)\gamma(r_2)$.
Altogether, applying this with $t=M$, $r_1=N$ and $r_2=1/e^{-\tau}$, we get that
\begin{equation*}
    \alpha_k(\hilbertianEmbedBound(M,e^{-\tau}/N)) \leq \beta_\hilbertianEmbed(M, N/e^{-\tau}) \leq \hat\rho(M)\gamma(N)\gamma(1/e^{-\tau}),
\end{equation*}
which is exactly a bound of the type asked for in the result above.
\end{remark}
\section{A new feature space for Gaussian kernels on Hilbert spaces}
\subsection{Technical background} \label{sec:distrslt:technicalBackgroundForFeatureSpace}
For a measure space $(X,\mu)$, $p\in[1,\infty)$, and $\K\in\{\R,\C\}$, we denote by $L^p(X,\mu;\K)$ the usual Lebesgue space of $\mu$-a.e. equivalence classes of $p$-integrable $\K$-valued functions.
We denote by $L^p_\R(X,\mu;\C)$ the real Banach space arising from the complex Banach space $L^p(X,\mu;\C)$ when restricting scalar multiplication to the reals.

We now recall a few basic facts related to probability measures on separable Hilbert spaces, following \cite{daprato2002second,daprato2006introduction}.
Let $\hs{H}$ be a separable real Hilbert space. %
Let $L(\hs{H})$ denote the set of continuous linear maps $T:\hs{H}\rightarrow\hs{H}$,
let $L_1(\hs{H})$ denote the set of trace-class linear operators on $\hs{H}$,
and for $T \in L_1(\hs{H})$, we denote by $\tr(T)$ its trace.
Recall that since $\hs{H}$ is a Hilbert space, a bounded linear operator is trace-class if and only if it is nuclear.
Furthermore, we denote by $L^+(\hs{H})$ the set of continuous linear operators that are self-adjoint ($\langle Tx,y\rangle_\hs{H}=\langle x,Ty\rangle_\hs{H}$) and positive ($\langle x, Tx\rangle_\hs{H} \in \Rnn$),
and define $L_1^+(\hs{H})=L_1(\hs{H})\cap L^+(\hs{H})$.

Let $\mu\in\distributions(\hs{H})$ be a Borel probability measure.
If $\int_\hs{H} \|x\|_\hs{H}\mathrm{d}\mu(x)<\infty$ (as a Lebesgue integral), then the Bochner integral $m=\int_\hs{H} x \mathrm{d}\mu(x)\in\hs{H}$ exists and is called the mean of $\mu$.
If $\int_\hs{H} \|x\|_\hs{H}^2\mathrm{d}\mu(x)<\infty$, then there exists a bounded linear operator $Q$ such that
\begin{equation*}
    \langle h_1, Q h_2\rangle_\hs{H} = \int_\hs{H} \langle x - m, h_1\rangle_\hs{H} \langle x - m, h_2\rangle_\hs{H} \mathrm{d}\mu(x)
    \quad \forall h_1,h_2\in\hs{H}
\end{equation*}
holds. Moreover, we actually have $Q\in L_1^+(\hs{H})$, and we call $Q$ the covariance operator of $\mu$.

Finally, we introduce some facts about Gaussian measures on separable Hilbert spaces.
The following result provides existence and uniqueness of Gaussian measures using the Fourier transform, and it corresponds to \cite[Theorem~1.12]{daprato2006introduction}.
\begin{proposition} \label{prop:gaussian:existenceUniquenessFourier}
For all $a\in\hs{H}$ and $Q\in L_1^+(\hs{H})$ there exists exactly one Borel probability measure $\mu\in\distributions(\hs{H})$ with
\begin{equation}
    \int_\hs{H} \exp(i\langle h,x\rangle) \mathrm{d}\mu(x) = \exp(i\langle a, h\rangle_\hs{H})\exp\left(-\frac12 \langle Qh, h\rangle_\hs{H}\right)
    \quad \forall h\in\hs{H}
\end{equation}
This probability measure has mean $a$, covariance operator $Q$, and we denote it by $\Normal(a,Q)$.
\end{proposition}

\subsection{The white noise mapping}
We now recall the \emph{white noise mapping} of a Gaussian measure, which is a map from $\hs{H}$ into $L^2(\hs{H},\Normal(0,Q);\R)$.
Our exposition closely follows \cite[Section~1.2.4]{daprato2002second}.
\begin{lemma} \label{lemma:kernelDense}
Let $Q\in L_1^+(\hs{H})$. If $\ker(Q)=\{0\}$, then $Q^\frac12(\hs{H})$ is dense in $\hs{H}$.
\end{lemma}
\begin{proof}
Since $\hs{H}$ is a Hilbert space, a vector subspace $A\subseteq\hs{H}$ is dense if and only if $A^\perp=\{0\}$.
Let $x_0 \in Q^\frac12(\hs{H})^\perp$. For all $x\in \hs{H}$ we then have
\begin{equation*}
    0 = \langle x_0, Q^\frac12 x\rangle_\hs{H} = \langle Q^\frac12 x_0, x\rangle_\hs{H},
\end{equation*}
which implies that $Q^\frac12 x_0 = 0$, which shows that
\begin{equation*}
    Qx_0 = Q^\frac12 Q^\frac12x_0 = Q^\frac12 0 = 0,
\end{equation*}
so $x_0\in\ker(Q)$, hence by assumption $x_0=0$. 
Altogether, we get that $Q^\frac12(\hs{H})^\perp=\{0\}$, establishing the claim.
\end{proof}
Let $Q\in L_1^+(\hs{H})$, then $Q^\frac12(\hs{H})$ is called the kernel of $\Normal(0,Q)$. 
Assume from now on that $\ker(Q)=\{0\}$, so that $Q^\frac12(\hs{H})$ is dense.
Furthermore, we also have $\ker(Q^\frac12)=\{0\}$ (if $Q^\frac12 x = 0$, then $Qx=Q^\frac12 Q^\frac12x=Q^\frac12 0=0$, so $\ker(Q^\frac12)\subseteq \ker(Q)=\{0\}$),
so $Q^\frac12$ is also injective.
Since the inverse of $Q^\frac12$ exists on $Q^\frac12(\hs{H})$, we can define a linear map
\begin{equation}
    Q^\frac12(\hs{H}) \rightarrow L^2(\hs{H},\Normal(0,Q);\R),
    \quad 
    h \mapsto [\langle Q^{-\frac12}h,\cdot\rangle_{\hs{H}}],
\end{equation}
where $[f]$ denotes the $\Normal(0,Q)$-nullset equivalence class of a $\Normal(0,Q)$-square integrable real-valued function $f$.
The linearity of this map is clear, and we have for all $h\in Q^{\frac12}(\hs{H})$ that
\begin{align*}
    \int_\hs{H} |[\langle Q^{-\frac12}h,x\rangle_{\hs{H}}]|^2\mathrm{d}\Normal(x\mid 0,Q)
    & = \int_\hs{H} |\langle Q^{-\frac12}h,x\rangle_{\hs{H}}|^2\mathrm{d}\Normal(x\mid 0,Q) \\
    & \leq \| Q^{-\frac12}h\|_\hs{H}^2 \int_\hs{H}\|x\|_{\hs{H}}^2\mathrm{d}\Normal(x\mid 0,Q) <\infty
\end{align*}
so the range of this map consists of square-integrable function equivalence classes, i.e., its range lies indeed in $L^2(\hs{H},\Normal(0,Q);\R)$.
Furthermore, for all $h_1,h_2\in Q^{\frac12}(\hs{H})$ we have
\begin{align*}
    \langle [\langle Q^{-\frac12}h_1,\cdot\rangle_{\hs{H}}] [\langle Q^{-\frac12}h_2,\cdot\rangle_{\hs{H}}] \rangle_{L^2(\hs{H},\Normal(0,Q);\R)}
    & = \int_\hs{H} [\langle Q^{-\frac12}h_1,x\rangle_{\hs{H}}][\langle Q^{-\frac12}h_2,x\rangle_{\hs{H}}] \mathrm{d}\Normal(x\mid 0,Q) \\
    & = \int_\hs{H} \langle x, Q^{-\frac12}h_1\rangle_{\hs{H}}\langle x, Q^{-\frac12}h_2\rangle_{\hs{H}} \mathrm{d}\Normal(x\mid 0,Q) \\
    & = \langle  Q^{-\frac12}h_1, Q Q^{-\frac12} h_2 \rangle_\hs{H} \\
    & = \langle h_1, Q^{-\frac12}h_1 Q^{\frac12} Q^{\frac12} Q^{-\frac12} \rangle_\hs{H} \\
    & = \langle h_1, h_2 \rangle_\hs{H},
\end{align*}
where we used \cite[Proposition~1.2.4,~(1.2.6)]{daprato2002second} for the third equality.
This shows that this map is an isometry between $Q^{\frac12}(\hs{H})$ and $L^2(\hs{H},\Normal(0,Q);\R)$, so in particular, it is continuous. 
Since $Q^{\frac12}(\hs{H})$ is dense in $\hs{H}$, we can extend this map in a unique manner to a linear, continuous isometry on all of $\hs{H}$,
which we denote by $\hs{H} \ni h \mapsto W_h \in L^2(\hs{H},\Normal(0,Q);\R)$.
This map is called the \emph{white noise mapping} of the Gaussian measure $\Normal(0,Q)$.
\subsection{Proof of \Cref{thm:distrslt:featureSpaceGaussianKernelOnHS} (Feature Space for Gaussian Kernels on Hilbert spaces)} \label{sec:distrslt:proofOffeatureSpaceGaussianKernelOnHS}
For the proof we need some intermediate results, which follow using slightly modified arguments from the proof of  \cite[Proposition 1.2.7]{daprato2002second}.
\begin{lemma}
Let $\hs{H}$ be a separable real Hilbert space, $Q\in L_1^+(\hs{H})$ with $\ker(Q)=\{0\}$, and $\lambda\in\R$.
The map $\hs{H}\ni h \mapsto \exp(i W_h(\cdot)) \in L^2(\hs{H},\Normal(0,Q);\C)$ is well-defined and continuous.
\end{lemma}
\begin{proof}
Recall that since $\ker(Q)=\{0\}$, $Q^\frac12(\hs{H})$ is dense in $\hs{H}$.
Let $\lambda\in\R$ and $h\in Q^{\frac12}(\hs{H})$. We have
\begin{align*}
    \exp\left(-\frac{\lambda^2}{2}\|h\|_\hs{H}^2\right)
    & = 1 \cdot \exp\left(-\frac{1}{2}\langle \lambda Q^{\frac12}Q^{-\frac12}h, \lambda Q^{\frac12}Q^{-\frac12}h \rangle_\hs{H}\right)  \\
    & \stp{=}{1} \exp(i \langle 0, h\rangle_\hs{H}) 
        \exp\left(-\frac{1}{2}\langle \lambda Q \left(Q^{-\frac12}h\right), \lambda \left(Q^{-\frac12}h\right) \rangle_\hs{H}\right) \\
    & \stp{=}{2} \int_\hs{H} \exp(i\langle  \lambda Q^{-\frac12}h, x \rangle_\hs{H})\mathrm{d}\Normal(x\mid 0, Q) \\
    & = \int_\hs{H} \exp(i\lambda \langle Q^{-\frac12}h, x \rangle_\hs{H})\mathrm{d}\Normal(x\mid 0, Q) \\
    & \stp{=}{3} \int_\hs{H} \exp(i\lambda W_h(x))\mathrm{d}\Normal(x\mid 0, Q),
\end{align*}
where we used for \stpx{1} the fact that $Q^{\frac12}$ is self-adjoint and $e^0=1$,
in \stpx{2} we used the characterization of the Gaussian measure (via the Fourier transform) from Proposition \ref{prop:gaussian:existenceUniquenessFourier},
and in \stpx{3} the definition of $W_h(x)$ (recall that $h\in Q^{\frac12}(\hs{H})$ by assumption).
Note that since $|e^{is}|=1$ for all $s\in\R$, $\exp(i \lambda W_h(\cdot))\in L^2(\hs{H},\Normal(0,Q);\C)$.

Next, we show that for all $\lambda\in\R$, the map $Q^{\frac12}(\hs{H})\ni h \mapsto \exp(i \lambda W_h) \in L^2(\hs{H},\Normal(0,Q);\C)$ is continuous. 
Since $Q^{\frac12}(\hs{H})$ is a vector space, and $h \mapsto W_h$ is linear, it is enough to show this for $\lambda=1$.
Let $h_1,h_2\in Q^{\frac12}(\hs{H})$, then
\begin{align*}
    & \|\exp(iW_{h_1}(\cdot)) - \exp(iW_{h_2}(\cdot))\|_{L^2(\hs{H},\Normal(0,Q);\C)}^2
    = \int_\hs{H} \left|\exp(iW_{h_1}(x)) - \exp(iW_{h_2}(x))\right|^2\mathrm{d}\Normal(x\mid 0,Q) \\
    & \hspace{0.5cm} = \int_\hs{H} \left(\exp(iW_{h_1}(x)) - \exp(iW_{h_2}(x))\right)\overline{\left(\exp(iW_{h_1}(x)) - \exp(iW_{h_2}(x))\right)}\mathrm{d}\Normal(x\mid 0,Q) \\
    & \hspace{0.5cm} \stp{=}{1} \int_\hs{H} 1 - \exp(iW_{h_1}(x)-iW_{h_2}(x)) - \exp(iW_{h_2}(x)-iW_{h_1}(x)) + 1\mathrm{d}\Normal(x\mid 0,Q) \\
    & \hspace{0.5cm} \stp{=}{2} 2 - \int_\hs{H} \exp(i \langle Q^{-\frac12}(h_1-h_2),x\rangle_\hs{H})\mathrm{d}\Normal(x\mid 0,Q) - \int_\hs{H} \exp(i \langle Q^{-\frac12}(h_2-h_1),x\rangle_\hs{H})\mathrm{d}\Normal(x\mid 0,Q) \\
    & \hspace{0.5cm} \stp{=}{3} 2 
        -  \exp(i \langle 0, h\rangle_\hs{H}) \exp\left(-\frac{1}{2}\langle Q \left(Q^{-\frac12}(h_1-h_2)\right), \left(Q^{-\frac12}(h_1-h_2)\right) \rangle_\hs{H}\right) \\
        & \hspace{1cm} - \exp(i \langle 0, h\rangle_\hs{H}) \exp\left(-\frac{1}{2}\langle Q \left(Q^{-\frac12}(h_2-h_1)\right), \left(Q^{-\frac12}(h_2-h_1)\right) \rangle_\hs{H}\right) \\
    & \hspace{0.5cm} \stp{=}{4} 2 - \exp\left(-\frac12\|h_1-h_2\|_{\hs{H}}^2\right)
        - \exp\left(-\frac12\|h_2-h_1\|_{\hs{H}}^2\right) \\
    & \hspace{0.5cm} = 2 - 2\exp\left(-\frac12\|h_1-h_2\|_{\hs{H}}^2\right),
\end{align*}
In \stpx{1}, we used the usual rules $e^{ix}\overline{e^{ix}}=|e^{ix}|^2=1$
and $e^xe^y=e^{x+y}$,
and in \stpx{2} we used the linearity of the integral, the fact that $\Normal(0,Q)$ is a probability distribution, and the definition of $h \mapsto W_h$ (recall that $h_1,h_2\in Q^{\frac12}(\hs{H})$).
For \stpx{3} we used again the characterization of the Gaussian distribution via the Fourier transform from Proposition \ref{prop:gaussian:existenceUniquenessFourier},
and in \stpx{4} we used
\begin{equation*}
    \langle Q(Q^{-\frac12}h),(Q^{-\frac12}h)\rangle_\hs{H}
    = \langle Q^\frac12Q^\frac12(Q^{-\frac12}h),(Q^{-\frac12}h)\rangle_\hs{H}
    = \langle Q^\frac12(Q^{-\frac12}h), Q^\frac12(Q^{-\frac12}h) \rangle_\hs{H}
    = \|h\|_\hs{H}^2,
\end{equation*}
which holds for all $h\in Q^{\frac12}(\hs{H})$.
Observe now that for $h_2\rightarrow h_1$, the righthand side of the above equality chain converges to zero, showing the continuity of the map 
$Q^{\frac12}(\hs{H})\ni h \mapsto \exp(i W_h) \in L^2(\hs{H},\Normal(0,Q);\C)$,
and a fortiori of all the maps
$Q^{\frac12}(\hs{H})\ni h \mapsto \exp(i \lambda W_h) \in L^2(\hs{H},\Normal(0,Q);\C)$
with $\lambda\in\R$.
The result now follows due to denseness of $Q^\frac12(\hs{H})$ in $\hs{H}$.
\end{proof}
\begin{lemma} \label{lem:distrslt:expressionForGaussianKernelFunction}
    Let $\hs{H}$ be a separable real Hilbert space, $Q\in L_1^+(\hs{H})$ with $\ker(Q)=\{0\}$, and $\lambda\in\R$.
    For all $\lambda\in\R$ and $h\in\hs{H}$ we have
    \begin{equation} \label{eq:expRepresentationWNM}
        \exp\left(-\frac{\lambda^2}{2}\|h\|_\hs{H}^2\right) = \int_\hs{H} \exp(i\lambda W_h(x))\mathrm{d}\Normal(x\mid 0, Q).
   \end{equation}
\end{lemma}
\begin{proof}
    Let $\lambda\in\R$ and $h\in \hs{H}$ be arbitrary.
    By denseness of $Q^\frac12(\hs{H})$ in $\hs{H}$, there exists a sequence $(h_n)_n\subseteq Q^{\frac12}(\hs{H})$ such that $h_n\rightarrow h$, and we get
    \begin{align*}
        & \left|  \exp\left(-\frac{\lambda^2}{2}\|h\|_\hs{H}^2\right) - \int_\hs{H} \exp(i\lambda W_h(x))\mathrm{d}\Normal(x\mid 0, Q)\right|
            \stp{\leq}{1}  \left|  \exp\left(-\frac{\lambda^2}{2}\|h\|_\hs{H}^2\right) -  \exp\left(-\frac{\lambda^2}{2}\|h_n\|_\hs{H}^2\right) \right| \\
            & \hspace{1cm} + \left| \exp\left(-\frac{\lambda^2}{2}\|h_n\|_\hs{H}^2\right) - \int_\hs{H} \exp(i\lambda W_{h_n}(x))\mathrm{d}\Normal(x\mid 0, Q) \right| \\
            & \hspace{1cm} + \left|\int_\hs{H} \exp(i\lambda W_{h_n}(x))\mathrm{d}\Normal(x\mid 0, Q) - \int_\hs{H} \exp(i\lambda W_h(x))\mathrm{d}\Normal(x\mid 0, Q)\right| \\
        & \hspace{0.5cm} \stp{=}{2} \left|  \exp\left(-\frac{\lambda^2}{2}\|h\|_\hs{H}^2\right) -  \exp\left(-\frac{\lambda^2}{2}\|h_n\|_\hs{H}^2\right) \right| \\
            & \hspace{1cm} + \left|\int_\hs{H} \exp(i\lambda W_{h_n}(x))\mathrm{d}\Normal(x\mid 0, Q) - \int_\hs{H} \exp(i\lambda W_h(x))\mathrm{d}\Normal(x\mid 0, Q)\right| \\
        & \hspace{0.5cm} \stp{\leq}{3}  \left|  \exp\left(-\frac{\lambda^2}{2}\|h\|_\hs{H}^2\right) - \exp\left(-\frac{\lambda^2}{2}\|h_n\|_\hs{H}^2\right) \right| 
            + \int_\hs{H} |\exp(i\lambda W_{h_n}(x)) - \exp(i\lambda W_h(x))|\mathrm{d}\Normal(x\mid 0, Q) \\
        & \hspace{0.5cm} = \left|  \exp\left(-\frac{\lambda^2}{2}\|h\|_\hs{H}^2\right) - \exp\left(-\frac{\lambda^2}{2}\|h_n\|_\hs{H}^2\right) \right| 
            + \| \exp(i\lambda W_{h_n}(\cdot)) - \exp(i\lambda W_h(\cdot))\|_{L^1(\hs{H},\Normal(0,Q);\C)} \\   
        & \hspace{0.5cm} \stp{\leq}{4} \left|  \exp\left(-\frac{\lambda^2}{2}\|h\|_\hs{H}^2\right) - \exp\left(-\frac{\lambda^2}{2}\|h_n\|_\hs{H}^2\right) \right| 
            + \| \exp(i\lambda W_{h_n}(\cdot)) - \exp(i\lambda W_h(\cdot))\|_{L^2(\hs{H},\Normal(0,Q);\C)} \\
        & \hspace{0.5cm} \rightarrow 0,
    \end{align*}
    where we used the triangle inequality for \stpx{1},
    for \stpx{2} we used the previously established fact that $\exp\left(-\frac{\lambda^2}{2}\|h_n\|_\hs{H}^2\right) = \int_\hs{H} \exp(i\lambda W_{h_n}(x))\mathrm{d}\Normal(x\mid 0, Q)$ (since $h_n\in Q^{\frac12}(\hs{H})$),
    in \stpx{3} the triangle inequality was used again,
    and in \stpx{4} we used the fact that the Gaussian distribution is a probability distribution.
    Finally, the continuity of $h \mapsto \exp(i\lambda W_h(\cdot))$ implies convergence above.
    Altogether, this establishes the claim.
\end{proof}
After these preparations, we can now turn to feature spaces of Gaussian kernels on separable Hilbert spaces.
The following result is already new to the best of our knowledge.
\begin{lemma} \label{lem:distrslt:complexFeatureSpaceForGaussianKernel}
    Let $\hs{H}$ be a separable real Hilbert space and $k_\gamma$ the Gaussian kernel on $\hs{H}$ with length scale $\gamma$.
    A feature space for $k_\gamma$ is given by $L^2(\hs{H},\Normal(0,Q);\C)$, and 
    \begin{equation}
        \fm_Q^\C: \hs{H}\rightarrow L^2(\hs{H},\Normal(0,Q);\C),
        \quad
        \fm_Q^\C(x) = \exp(i \sqrt{2}/\gamma \cdot W_x(\cdot)).
    \end{equation}
    is a corresponding feature map.
\end{lemma}
Note that although $k$ is real-valued, the preceding results gives a complex Hilbert space as a feature space, cf. \cite[Section~4.1]{SC08} for some more details on this issue.
\begin{proof}
    Recall that the Gaussian kernel on $\hs{H}$ is given by
    \begin{equation*}
        k_\gamma(h_1,h_2) = \exp\left(-\frac{\|h_1-h_2\|^2_\hs{H}}{\gamma^2}\right).
    \end{equation*}
    Let now $h_1,h_2\in Q^{\frac12}(\hs{H})$, and note that since $Q^{\frac12}(\hs{H})$ is a subvectorspace, $h_1-h_2\in Q^{\frac12}(\hs{H})$.
    Defining $\lambda=\frac{\sqrt{2}}{\gamma}$ for brevity, we have
    \begin{align*}
        k_\gamma(h_1,h_2) & = \exp\left(-\frac{\|h_1-h_2\|_\hs{H}^2}{\gamma^2}\right) \\
        & =  \exp\left(-\frac12 \left(\frac{\sqrt{2}}{\gamma}\right)^2 \|h_1-h_2\|_\hs{H}^2\right) \\
        & \stp{=}{1} \int_\hs{H} \exp(i \lambda W_{h_1-h_2}(x))\mathrm{d}\Normal(x\mid 0, Q) \\
        & \stp{=}{2} \int_\hs{H} \exp(i \lambda \langle Q^{-\frac12}(h_1-h_2), x\rangle_\hs{H}) \mathrm{d}\Normal(x\mid 0, Q) \\
        & = \int_\hs{H} \exp(i \lambda \langle Q^{-\frac12}h_1,x\rangle_\hs{H}) \exp(-i\lambda \langle Q^{-\frac12}h_2,x\rangle_\hs{H}) \mathrm{d}\Normal(x\mid 0, Q) \\
        & = \int_\hs{H} \exp(i \lambda W_{h_1}(x)) \overline{\exp(i\lambda W_{h_2}(x))} \mathrm{d}\Normal(x\mid 0, Q) \\
        & = \langle  \exp(i \sqrt{2}/\gamma \cdot W_{h_1}(\cdot)),  \exp(i \sqrt{2}/\gamma \cdot W_{h_2}(\cdot))\rangle_{L^2(\hs{H},\Normal(0,Q);\C)},
    \end{align*}
    where we used Lemma \ref{lem:distrslt:expressionForGaussianKernelFunction} for \stpx{1}, 
    and the fact that $h_1,h_2\in Q^{\frac12}(\hs{H})$ and the explicit formula for the white noise mapping in \stpx{2}.
    Using the density of $Q^{\frac12}(\hs{H})$ in $\hs{H}$ and the continuity of  $h \mapsto \exp(iW_h(\cdot))$,
    we can extend this to all of $\hs{H}$, i.e., for all $x_1,x_2\in\hs{H}$ we have
    \begin{equation*}
        k(x_1,x_2) = \langle  \exp(i \lambda W_{x_1}(\cdot)),  \exp(i \lambda W_{x_2}(\cdot))\rangle_{L^2(\hs{H},\Normal(0,Q);\C)}.
    \end{equation*}
    This shows that indeed $L^2(\hs{H},\Normal(0,Q);\C)$ is a complex feature space of $k$, and $\fm_Q^\C$ a corresponding feature map.
\end{proof}
We are finally ready to introduce the new feature space for the Gaussian kernel on real separable Hilbert spaces,
and we start with the case that the kernel is defined on all of the input Hilbert space.
\begin{proposition} \label{prop:distrslt:featureSpaceGaussianKernelOnAllofHS}
Let $\hs{H}$ be a separable real Hilbert space and $\gaussianKernel{\gamma}$ the Gaussian kernel on $\hs{H}$ with length scale $\gamma\in\Rp$.
For all $Q\in L_1^+(\hs{H})$ with $\ker(Q)=\{0\}$, $L^2_\R(\hs{H},\Normal(0,Q);\C)$ is a (real) feature space of $k_\gamma$,
\ifTwoColumns
$\fm_Q: \hs{H}\rightarrow L^2_\R(\hs{H},\Normal(0,Q);\C)$ defined by
\begin{equation}
    \fm_Q(x) = \exp(i \sqrt{2}/\gamma \cdot W_x(\cdot))
\end{equation}
\else
\begin{equation}
    \fm_Q: \hs{H}\rightarrow L^2_\R(\hs{H},\Normal(0,Q);\C),
    \quad
    \fm_Q(x) = \exp(i \sqrt{2}/\gamma \cdot W_x(\cdot))
\end{equation}
\fi
is a corresponding feature map, 
and the canonical surjection $V_Q: L^2_\R(\hs{H},\Normal(0,Q);\C) \rightarrow H_\gamma$ is given by
\begin{equation}
    (V_Qg)(x) = \Re \int_\hs{H} \exp\left(-i \frac{\sqrt{2}}{\gamma} W_x(z)\right)g(z)\mathrm{d}\Normal(z\mid 0, Q).
\end{equation}
\end{proposition}
\begin{proof}
Recall that $L^2_\R(\hs{H},\Normal(0,Q);\C)$ is defined as the $\R$-vector space arising from $L^2(\hs{H},\Normal(0,Q);\C)$ when restricting scalar multiplication to $\R$.
Furthermore, define for $f,g\in L^2_\R(\hs{H},\Normal(0,Q);\C)$
\begin{equation}
    \langle f, g\rangle_{L^2_\R(\hs{H},\Normal(0,Q);\C)} = \Re\left( \langle f, g\rangle_{L^2(\hs{H},\Normal(0,Q);\C)}\right).
\end{equation}
Now, $(L^2_\R(\hs{H},\Normal(0,Q);\C), \scp_{L^2_\R(\hs{H},\Normal(0,Q);\C)})$ is a real Hilbert space,
and $L^2(\hs{H},\Normal(0,Q);\C)$ is a complex feature space of $k_\gamma$ with feature map $\fm_Q^\C$ according to Lemma \ref{lem:distrslt:complexFeatureSpaceForGaussianKernel},
and so \cite[Lemma~4.4]{SC08} ensures that $L^2_\R(\hs{H},\Normal(0,Q);\C)$ is a real feature space of $k_\gamma$, and 
\begin{equation}
    \fm_Q: \hs{H}\rightarrow L^2_\R(\hs{H},\Normal(0,Q);\C),
    \quad
    \fm_Q(x) = \exp(i \sqrt{2}/\gamma \cdot W_x(\cdot)).
\end{equation}
is a corresponding feature map.
Furthermore, the canonical embedding
$V_Q: L^2_\R(\hs{H},\Normal(0,Q);\C) \rightarrow H_\gamma$ is given by
\begin{align*}
    (V_Qg)(x) & = \langle g, \fm_Q(x)\rangle_{L^2_\R(\hs{H},\Normal(0,Q);\C)}\\
    & = \Re  \langle g, \fm_Q(x)\rangle_{L^2(\hs{H},\Normal(0,Q);\C)} \\
    & = \Re \int_\hs{H} \exp\left(-i \frac{\sqrt{2}}{\gamma} W_x(z)\right)g(z)\mathrm{d}\Normal(z\mid 0, Q),
\end{align*}
cf. \cite[Theorem~4.21]{SC08}.
\end{proof}
Theorem \ref{thm:distrslt:featureSpaceGaussianKernelOnHS} now follows as an immediate corollary of the preceding result.
\begin{proof}[Proof of Theorem \ref{thm:distrslt:featureSpaceGaussianKernelOnHS}]
Define the Gaussian kernel on all of $\hs{H}$ and length scale $\gamma$ by
\begin{equation*}
    k_\gamma^\hs{H}(h,h') = \exp\left(-\frac{\|h-h'\|_\hs{H}}{\gamma^2}\right),
\end{equation*}
then $k_\gamma=k_\gamma^\hs{H}\lvert_{X \times X}$.
According to Proposition \ref{prop:distrslt:featureSpaceGaussianKernelOnAllofHS}, $L^2_\R(\hs{H},\Normal(0,Q);\C)$ is real feature space of $k_\gamma^\hs{H}$ with feature map 
\begin{equation*}
    \fm_Q^\hs{H}: \hs{H}\rightarrow L^2_\R(\hs{H},\Normal(0,Q);\C),
    \quad
    \fm_Q^\hs{H}(x) = \exp(i \sqrt{2}/\gamma \cdot W_x(\cdot)),
\end{equation*}
and by definition we have $\fm_Q = \fm_Q^\hs{H}\lvert_X$.
Since for all $x,x'\in\inputSet\subseteq\hs{H}$ we have
\begin{align*}
    k_\gamma(x,x') & = k_\gamma^\hs{H}(x,x') \\
    & = \langle  \fm_Q^\hs{H}\lvert_X(x'),  \fm_Q^\hs{H}\lvert_X(x) \rangle_{L^2_\R(\hs{H},\Normal(0,Q);\C)} \\
    & = \langle \fm_Q(x'),  \fm_Q(x) \rangle_{L^2_\R(\hs{H},\Normal(0,Q);\C)}, 
\end{align*}
$L^2_\R(\hs{H},\Normal(0,Q);\C)$ is also a feature space for $k_\gamma$, 
with feature map $\fm_Q$.
\end{proof}
Since the feature map for the Gaussian kernel on $X \subseteq\hs{H}$ arises from the feature map for the Gaussian kernel on all of $\hs{H}$ by restriction to $X$, we use the same symbol for both feature maps.
\section{Technical Details for Distributional Classification with Gaussian Kernels}
\subsection{Proof of Theorem \ref{thm:distrslt:boundApproxErrorFuncGaussianKernelHingeLoss} (Bound on the approximation error function)} 
\label{sec:distrslt:proofOfboundApproxErrorFuncGaussianKernelHingeLoss}
\textbf{Defining the candidate hypothesis}
Let $f_P^\ast$ be the function defined in Section \ref{sec:distrslt:learningRatesClassification} that achieves the Bayes risk, so
\begin{equation*}
    \Risk{\ellClass}{P}(f_P^\ast) = \RiskBayes{\ellClass}{P},
\end{equation*}
and define for $z\in \hs{H}$
\begin{equation}
    \hat g(z) = \int_\hs{H} \exp\left(i \frac{\sqrt{2}}{\gamma} W_y(z)\right) f_P^\ast(y) \mathrm{d}\Normal(y\mid 0,Q).
\end{equation}
Since the absolute value of the integrand is bounded by 1, and $\Normal(y\mid 0,Q)$ is a probability distribution, $\hat g(z)$ is well-defined.
Furthermore, Fubini's theorem and
\begin{align*}
    \int_{\hs{H}} |\hat g(z)|^2 \mathrm{d}\Normal(z\mid 0,Q) 
    & =  \int_{\hs{H}} \left|
        \int_\hs{H} \exp\left(i \frac{\sqrt{2}}{\gamma} W_y(z)\right) f_P^\ast(y) \mathrm{d}\Normal(y\mid 0,Q)
        \right|^2\mathrm{d}\Normal(z\mid 0,Q) \\
    & \leq \int_{\hs{H}} \left(
        \int_\hs{H} |\exp\left(i \frac{\sqrt{2}}{\gamma} W_y(z)\right) f_P^\ast(y)| \mathrm{d}\Normal(y\mid 0,Q)
        \right)^2\mathrm{d}\Normal(z\mid 0,Q) \\
    & \leq \int_{\hs{H}} 1 \mathrm{d}\Normal(z\mid 0,Q) \\
    & = 1
\end{align*}
show that $\hat g \in L^2_\R(\hs{H},\Normal(0,Q);\C)$ with $\|\hat g\|_{L^2_\R(\hs{H},\Normal(0,Q);\C)}^2 \leq 1$.

Our candidate hypothesis will be $\hat f = V_Q \hat g \in H_\gamma$.
Since $V_Q$ is the canonical surjection from \Cref{thm:distrslt:featureSpaceGaussianKernelOnHS}, we have
\begin{equation*}
    \|\hat f\|_{k_\gamma} = \|V_Q \hat g\|_{k_\gamma} \leq \|\hat g\|_{L^2_\R(\hs{H},\Normal(0,Q);\C)}^2 \leq 1.
\end{equation*}
For later use, we also record that we have for all $x\in\inputSet$
\begin{align*}
    (V_Q \hat g)(x) & \stp{=}{1} \Re \int_\hs{H} \exp\left(-i \frac{\sqrt{2}}{\gamma} W_x(z) \right)\hat{g}(z)\mathrm{d}\Normal(z\mid 0,Q) \\
    & \stp{=}{2} \Re \int_\hs{H} \exp\left(-i \frac{\sqrt{2}}{\gamma} W_x(z) \right)
        \int_\hs{H} \exp\left(i \frac{\sqrt{2}}{\gamma} W_y(z)\right) f_P^\ast(y) 
        \mathrm{d}\Normal(y\mid 0,Q)\mathrm{d}\Normal(z\mid 0,Q) \\
    & \stp{=}{3} \Re \int_\hs{H}  
        \int_\hs{H} \exp\left(-i \frac{\sqrt{2}}{\gamma} W_{x-y}(z) \right)f_P^\ast(y) \mathrm{d}\Normal(y\mid 0,Q)
        \mathrm{d}\Normal(z\mid 0,Q) \\
    & \stp{=}{4} \Re \int_\hs{H} f_P^\ast(y) 
        \int_\hs{H}\exp\left(i \frac{\sqrt{2}}{\gamma} W_{y-x}(z) \right)\mathrm{d}\Normal(z\mid 0,Q) 
        \mathrm{d}\Normal(y\mid 0,Q) \\
    & \stp{=}{5}  \Re \int_\hs{H} f_P^\ast(y) 
        \exp\left(-\frac{\|x-y\|_\hs{H}^2}{\gamma^2}\right)
        \mathrm{d}\Normal(y\mid 0,Q) \\
    & = \int_\hs{H} f_P^\ast(y) 
        \exp\left(-\frac{\|x-y\|_\hs{H}^2}{\gamma^2}\right)
        \mathrm{d}\Normal(y\mid 0,Q).
\end{align*}
For \stpx{1} we used the definition of $V_Q$,
for \stpx{2} the choice of $\hat f$,
for \stpx{3} the linearity of the integral, $e^xe^y=e^{x+y}$, and the linearity of the white noise mapping,
for \stpx{4} Fubini and the linearity of the integral again,
for \stpx{5} we used \eqref{eq:expRepresentationWNM},
and finally that the integral is real-valued.

Furthermore,
\begin{align*}
    |(V_Q \hat g)(x)| & = \left|  \Re \int_\hs{H} \exp\left(-i \frac{\sqrt{2}}{\gamma} W_x(z)\right)\hat g(z)\mathrm{d}\Normal(z\mid 0, Q) \right| \\
    & \leq \left|  \int_\hs{H} \exp\left(-i \frac{\sqrt{2}}{\gamma} W_x(z)\right)\hat g(z)\mathrm{d}\Normal(z\mid 0, Q) \right| \\
    & \leq  \int_\hs{H} \left|\exp\left(-i \frac{\sqrt{2}}{\gamma} W_x(z)\right)\right| |\hat g(z)|\mathrm{d}\Normal(z\mid 0, Q) \\
    & \leq 1,
\end{align*}
where we used in the last step that $\left|\exp\left(-i \frac{\sqrt{2}}{\gamma} W_x(z)\right)\right|, |\hat g(z)|\leq 1$ for all $z$, and that $\Normal(0, Q)$ is a probability distribution.

\textbf{Bounding the pointwise risk}
Let now $x \in X_1$, then
\begin{align*}
    (V_Q \hat g)(x) & \stp{=}{1}  \int_\hs{H}  
        \exp\left(-\frac{\|x-y\|_\hs{H}^2}{\gamma^2}\right) f_P^\ast(y) 
        \mathrm{d} \Normal(y\mid 0,Q) \\
    & \stp{=}{2}  \int_\hs{H}  
        \exp\left(-\frac{\|x-y\|_\hs{H}^2}{\gamma^2}\right) (f_P^\ast(y)+1)
        \mathrm{d} \Normal(y\mid 0,Q)
        -
        \int_\hs{H} \exp\left(-\frac{\|x-y\|_\hs{H}^2}{\gamma^2}\right) \mathrm{d} \Normal(y\mid 0,Q) \\
    & \stp{\geq}{3} \int_{X_1} 
        \exp\left(-\frac{\|x-y\|_\hs{H}^2}{\gamma^2}\right) (f_P^\ast(y)+1)
        \mathrm{d} \Normal(y\mid 0,Q)
        -
        \int_\hs{H} \exp\left(-\frac{\|x-y\|_\hs{H}^2}{\gamma^2}\right) \mathrm{d} \Normal(y\mid 0,Q) \\
    & \stp{=}{4} 2\int_{X_1} 
        \exp\left(-\frac{\|x-y\|_\hs{H}^2}{\gamma^2}\right)
        \mathrm{d} \Normal(y\mid 0,Q)
        -
        \int_\hs{H} \exp\left(-\frac{\|x-y\|_\hs{H}^2}{\gamma^2}\right) \mathrm{d} \Normal(y\mid 0,Q) \\
    & \stp{\geq}{5} 2\int_{B_{\Delta(x)}(x)}
        \exp\left(-\frac{\|x-y\|_\hs{H}^2}{\gamma^2}\right)
        \mathrm{d} \Normal(y\mid 0,Q)
        -
        \int_\hs{H} \exp\left(-\frac{\|x-y\|_\hs{H}^2}{\gamma^2}\right) \mathrm{d} \Normal(y\mid 0,Q),
\end{align*}
where we used the calculations above for \stpx{1},
the linearity of the integral for \stpx{2},
the nonnegativity of the integrand of the first integral for \stpx{3},
the fact that $f_P^\ast\lvert_{X_1}\equiv 1$ for \stpx{4},
and finally the fact that $B_{\Delta(x)}(x)\subseteq X_1$ (since $x\in X_1$ by assumption) for \stpx{5}.
We then get
\begin{align*}
    |(V_Q \hat g)(x) - f_P^\ast(x)| & = 1 - (V_Q \hat g)(x) \\
    & \leq 1 - 2\int_{B_{\Delta(x)}(x)}
        \exp\left(-\frac{\|x-y\|_\hs{H}^2}{\gamma^2}\right)
        \mathrm{d} \Normal(y\mid 0,Q)
        +
        \int_\hs{H} \exp\left(-\frac{\|x-y\|_\hs{H}^2}{\gamma^2}\right) \mathrm{d} \Normal(y\mid 0,Q).
\end{align*}
Let now $x\in X_{-1}$, then we have
\begin{align*}
    (V_Q \hat g)(x) & \stp{=}{1}  \int_\hs{H}  
        \exp\left(-\frac{\|x-y\|_\hs{H}^2}{\gamma^2}\right) f_P^\ast(y) 
        \mathrm{d} \Normal(y\mid 0,Q) \\
    & \stp{=}{2}  \int_\hs{H}  
        \exp\left(-\frac{\|x-y\|_\hs{H}^2}{\gamma^2}\right) (f_P^\ast(y)-1)
        \mathrm{d} \Normal(y\mid 0,Q)
        +
        \int_\hs{H} \exp\left(-\frac{\|x-y\|_\hs{H}^2}{\gamma^2}\right) \mathrm{d} \Normal(y\mid 0,Q) \\
    & \stp{\leq}{3} \int_{X_{-1}} 
        \exp\left(-\frac{\|x-y\|_\hs{H}^2}{\gamma^2}\right) (f_P^\ast(y)-1)
        \mathrm{d} \Normal(y\mid 0,Q)
        +
        \int_\hs{H} \exp\left(-\frac{\|x-y\|_\hs{H}^2}{\gamma^2}\right) \mathrm{d} \Normal(y\mid 0,Q) \\
    & \stp{=}{4} -2\int_{X_{-1}} 
        \exp\left(-\frac{\|x-y\|_\hs{H}^2}{\gamma^2}\right)
        \mathrm{d} \Normal(y\mid 0,Q)
        +
        \int_\hs{H} \exp\left(-\frac{\|x-y\|_\hs{H}^2}{\gamma^2}\right) \mathrm{d} \Normal(y\mid 0,Q) \\
    & \stp{\leq}{5} -2\int_{B_{\Delta(x)}(x)}
        \exp\left(-\frac{\|x-y\|_\hs{H}^2}{\gamma^2}\right)
        \mathrm{d} \Normal(y\mid 0,Q)
        +
        \int_\hs{H} \exp\left(-\frac{\|x-y\|_\hs{H}^2}{\gamma^2}\right) \mathrm{d} \Normal(y\mid 0,Q),
\end{align*}
where we used again the calculations from above for \stpx{1},
the linearity of the integral for \stpx{2},
the nonpositivity of the integrand of the first integral for \stpx{3},
the fact that $f_P^\ast\lvert_{X_{-1}}\equiv -1$ for \stpx{4},
and finally the fact that $B_{\Delta(x)}(x)\subseteq X_{-1}$ (since $x\in X_{-1}$ by assumption) for \stpx{5}.
We now get
\begin{align*}
    |(V_Q \hat g)(x) - f_P^\ast(x)| & = (V_Q \hat f)(x) - (-1) \\
    & \leq  1- 2\int_{B_{\Delta(x)}(x)}
        \exp\left(-\frac{\|x-y\|_\hs{H}^2}{\gamma^2}\right)
        \mathrm{d} \Normal(y\mid 0,Q)
        +
        \int_\hs{H} \exp\left(-\frac{\|x-y\|_\hs{H}^2}{\gamma^2}\right) \mathrm{d} \Normal(y\mid 0,Q).
\end{align*}
Summarizing, for all $x\in X_1 \cup X_{-1}$ we have
\begin{equation*}
        |(V_Q \hat g)(x) - f_P^\ast(x)| \leq 1- 2\int_{B_{\Delta(x)}(x)}
        \exp\left(-\frac{\|x-y\|_\hs{H}^2}{\gamma^2}\right)
        \mathrm{d} \Normal(y\mid 0,Q)
        +
        \int_\hs{H} \exp\left(-\frac{\|x-y\|_\hs{H}^2}{\gamma^2}\right) \mathrm{d} \Normal(y\mid 0,Q).
\end{equation*}
\textbf{Bounding the averaged excess risk}
Note that for $x\in X \setminus\{X_1\cup X_{-1}\}$ we have $2\eta(x)-1=0$, so using the above pointwise bounds in Zhang's Theorem \cite[Theorem~2.31]{SC08} results in
    {\scriptsize
\begin{align*}
        \Risk{\ellHinge}{P}(\hat f) - \RiskBayes{\ellHinge}{P}
    & =
    \int_X |(V_Q \hat g) (x) - f_P^\ast(x)||2\eta(x)-1|\mathrm{d}P_X(x) \\
    & =
        \int_{X_1 \cup X_{-1}} |(V_Q \hat g) (x) - f_P^\ast(x)||2\eta(x)-1|\mathrm{d}P_X(x) \\
    & \leq \int_{X_1 \cup X_{-1}} \left[ 
            1- 2\int_{B_{\Delta(x)}(x)}
            \exp\left(-\frac{\|x-y\|_\hs{H}^2}{\gamma^2}\right)
            \mathrm{d} \Normal(y\mid 0,Q)
            +
            \int_\hs{H} \exp\left(-\frac{\|x-y\|_\hs{H}^2}{\gamma^2}\right) \mathrm{d} \Normal(y\mid 0,Q)
        \right] \\
        & \hspace{1cm} \times |2\eta(x)-1|\mathrm{d}P_X(x) \\
    & =\int_{X_1 \cup X_{-1}} \left(
            1- 2\int_{B_{\Delta(x)}(x)}
            \exp\left(-\frac{\|x-y\|_\hs{H}^2}{\gamma^2}\right)
            \mathrm{d} \Normal(y\mid 0,Q)
        \right)  |2\eta(x)-1|\mathrm{d}P_X(x) \\
        & \hspace{1cm} +
        \int_{X_1 \cup X_{-1}} \left(
            \int_\hs{H} \exp\left(-\frac{\|x-y\|_\hs{H}^2}{\gamma^2}\right) \mathrm{d} \Normal(y\mid 0,Q)
        \right)|2\eta(x)-1|\mathrm{d}P_X(x) \\
    & =\int_{X_1 \cup X_{-1}} \left(
            1- 2\int_{B_{\Delta(x)}(x)}
            \exp\left(-\frac{\|x-y\|_\hs{H}^2}{\gamma^2}\right)
            \mathrm{d} \Normal(y\mid 0,Q)
        \right)  |2\eta(x)-1|\mathrm{d}P_X(x) \\
        & \hspace{1cm} +
        \int_\inputSet \left(
            \int_\hs{H} \exp\left(-\frac{\|y\|_\hs{H}^2}{\gamma^2}\right) \mathrm{d} \Normal(y\mid x,Q)
        \right)|2\eta(x)-1|\mathrm{d}P_X(x) \\
    & \leq 2C_Q \gamma^{2\alpha_Q},
\end{align*} }
where we used Assumption \ref{assumption:distrslt:geometricNoiseAssumptionGaussian} together with $\gamma^2 \leq \bar{t}_Q$ in the last inequality.

\textbf{Bounding the approximation error function}
By definition we have for all $\lambda\in\Rp$ that
\begin{align*}
    \ApproxErrorFunc{\ellHinge}{P}{H_k}(\lambda) & = \RiskRegOpt{\ellHinge}{P}{\lambda}{H_k} - \RiskOpt{\ellHinge}{P}{H_k} \\
        & = \left(\inf_{f\in H_k} \Risk{\ellHinge}{P}(f) + \lambda\|f\|_k^2 \right) - \inf_{f \in H_k} \Risk{\ellHinge}{P}(f) \\
        & \leq \Risk{\ellHinge}{P}(\hat f) + \lambda \|\hat f\|_k^2 - \RiskBayes{\ellHinge}{P} \\
        & =  2C_Q \gamma^{2\alpha_Q} + \lambda,
\end{align*}
establishing the claim.
\subsection{Proof of Theorem \ref{thm:distrslt:learningRateClassificationGaussianKernel} (Learning Rates for Gaussian Kernels and the Hinge Loss)} \label{sec:distrslt:proofOflearningRateClassificationGaussianKernel}
Combining Theorem \ref{thm:distrslt:oracleInequSVM} with Theorem \ref{thm:distrslt:boundApproxErrorFuncGaussianKernelHingeLoss}, we get that for all $\tau\geq 1$, with probability at least $1-4e^{-\tau}$, $ \Risk{\ell}{P}(\clipped{f}_{\dataSet_{\hat\hilbertianEmbed},\lambda} \circ \hilbertianEmbed) - \RiskBayes{\ell}{P}$
is upper bounded by 
\begin{align*}
    & C_1 \gamma^{2\beta_Q} + C_2\lambda + C_3N^{-1}\ln(N)\lambda^{-1} 
        + C_4\tau N^{-1}\gamma^{\beta_Q}\lambda^{-\frac12} + C_5\tau N^{-1} \\
    & \hspace{1cm} + \left( C_6 \gamma^{\beta_Q} \lambda^{-\frac12} + C_7 \lambda^{-\frac12} \right)
        M^{-\alpha/2}(1+\sqrt{\ln(N/e^{-\tau})})^\alpha
        + C_8 \tau N^{-\frac12}
\end{align*} 
for appropriate constants $C_1,\ldots,C_8$ (independent of $\lambda$, $N$, and $M$).
As in the proof of \Cref{thm:distrslt:learningRateSVMwithKMEs}, for a suitable constant $\tilde C$ we have
\begin{equation*}
    (1+\sqrt{\ln(N/e^{-\tau})}) \leq \tilde C(1+\tau+\ln(N)),
\end{equation*}
so we can find a constant $C\in\Rp$ (independent of $\lambda,N,M$) such that with probability at least $1-4e^{-\tau}$, $\Risk{\ell}{P}(\clipped{f}_{\dataSet_{\hat\hilbertianEmbed},\lambda} \circ \hilbertianEmbed) - \RiskBayes{\ell}{P}$ is upper bounded by
\begin{align*}
    C \tau \ln(N)\left(
       \gamma^{2\alpha_Q} + \lambda + N^{-1}\lambda^{-1} + N^{-1}\gamma^{\alpha_Q}\lambda^{-\frac12} + N^{-1}  
       + \left( \gamma^{\alpha_Q} \lambda^{-\frac12} + \lambda^{-\frac12} \right)M^{-\frac{\alpha}{2}} +  N^{-\frac12}
    \right),
\end{align*}
and choosing $M=N^{\frac{2}{\alpha}}$ leads to
\begin{align*}
    C \tau \ln(N)\left(
       \gamma^{2\alpha_Q} + \lambda + N^{-1}\lambda^{-1} + N^{-1}\gamma^{\alpha_Q}\lambda^{-\frac12} + N^{-1}  
       + \gamma^{\alpha_Q} \lambda^{-\frac12}N^{-1} + \lambda^{-\frac12} N^{-1} +  N^{-\frac12}
    \right).
\end{align*}
Assuming $\gamma,\lambda<1$ and adjusting $C$, we can upper bound this by
\begin{equation*}
    C \tau \ln(N)\left(
       \gamma^{2\alpha_Q} + \lambda + N^{-1}\lambda^{-1} + \gamma^{\alpha_Q} \lambda^{-\frac12}N^{-1} 
       +  N^{-\frac12}
    \right).
\end{equation*}
By setting $\lambda=N^{-\frac12}$ and adjusting $C$, this can be upper bounded by
\begin{equation*}
    C \tau \ln(N)\left(
       \gamma^{2\alpha_Q} + N^{-\frac12} + \gamma^{\alpha_Q} N^{-\frac12}
    \right).
\end{equation*}
Setting now $\gamma=N^{-\mu}$ and adjusting once again $C$, we get a new upper bound
\begin{equation*}
        C \tau \ln(N)\left( N^{-2\mu\alpha_Q} + N^{-\frac12} + N^{-\mu\alpha_Q-\frac12}
    \right),
\end{equation*}
and by adjusting $C$ a final time, we arrive at the upper bound
\begin{equation*}
     C \tau \ln(N) N^{-\min\{2\mu\alpha_Q, \frac12\}}.
\end{equation*}
Rescaling $\tau$ then establishes the result.
\subsection{Discussion of Assumption \ref{assumption:distrslt:geometricNoiseAssumptionGaussian}} \label{sec:distrslt:discussionOfgeometricNoiseAssumptionGaussian}
Since Assumption \ref{assumption:distrslt:geometricNoiseAssumptionGaussian} is novel and might look somewhat opaque, we provide a detailed discussion of its interpretation here.
Furthermore, we would like to stress that while the particular form of this assumption is new,
conceptually it is very similar to classic notations in the theory of binary classification.
In particular, one can interpet Assumption \ref{assumption:distrslt:geometricNoiseAssumptionGaussian} as variant of the \emph{geometric noise exponent} assumption \cite[Definition~2.3]{steinwart2007fast}, adapted to the present setting of subsets of Hilbert spaces as input spaces.

First, we need some preliminary considerations. Recall that in the framework of statistical learning theory, the overall goal in binary classification is to find a hypothesis $f: \inputSet \rightarrow \R$ such that
\begin{equation*}
    \Risk{\ell_c}{P}(f) = \int_{\inputSet\times\{-1,1\}} \ell_c(y,f(x)) \mathrm{d}P(x,y)
\end{equation*}
is small, where $\ellClass$ is the zero-one-loss.
The best one can do is achieving the Bayes risk, and as is well-known,
\begin{equation*}
    \RiskBayes{\ellClass}{P} = \int_\inputSet \min\{\eta(x), 1-\eta(x)\} \mathrm{d}P_X(x),
\end{equation*}
where $\eta: \inputSet\rightarrow[0,1]$ is a version of the conditional probability $P[Y=1 \mid X=x]$,
cf. \cite[Chapters~2,~8]{SC08}.
Furthermore, defining  
\begin{equation*}
    f_P^\ast(x) = \begin{cases}
        1 & \text{if } \eta(x)\geq \frac12 \\
        -1 & \text{otherwise}
    \end{cases}
\end{equation*}
we have $\Risk{\ell_c}{P}(f_P^\ast)=\RiskBayes{\ellClass}{P}$,
so this particular hypothesis achieves the Bayes risk.
A short calculation shows that
\begin{equation*}
        \Pb_{Y \mid X=x}[f_P^\ast(x) \text{ is wrong}] = \min\{\eta(x),1-\eta(x)\},
\end{equation*}
so the risk of $f_P^\ast$ (the Bayes risk) is just the probability of misclassification of a given input $x$, averaged over all inputs according to $P_X$.

If $\eta(x)=\{0,1\}$, then classification at $x$ is essentially deterministic, i.e., the label $y$ at $x$ is perfectly predictable.
If $\eta(x)=\frac12$, then the label $y$ is completely random, and we cannot do better than just guessing, so in this case we have zero predictability of the label $y$ at input $x$.

Finally, observe that
\begin{equation*}
    |2\eta(x)-1| = 1 - 2\min\{\eta(x), 1-\eta(x)\},
\end{equation*}
so we can interpret $|2\eta(x)-1|$ as the \emph{normalized predictability} of the label $y$ at input $x$. 
Indeed, for $\eta(x)=\frac12$ (we can only guess, no predictability), we have $|2\eta(x)-1|=1-2\cdot\frac12=0$, and for $\eta(x)\in\{0,1\}$ (deterministic case, perfect predictability) we have $|2\eta(x)-1|=1$.

Now we can turn to \eqref{eq:assumption:distrslt:geometricNoiseAssumptionGaussian:1}.
Define weighting terms for $t\in\Rp$ and $x\in\hs{H}$ by
\begin{equation}
    \phi_{t,Q}(x) = 1- 2\int_{B_{\Delta(x)}(x)}
            \exp\left(-\frac{\|x-y\|_\hs{H}^2}{t}\right)
            \mathrm{d} \Normal(y\mid 0,Q).
\end{equation}
It is reasonable to assume that $\phi_{t,Q}(x)\geq 0$ for all $0< t < \bar t$ (after all, $\exp(-s/t)$ is integrated over a bounded set, so $t \searrow 0$ forces this term towards zero).
Let us interpret this weighting term.
The term
\begin{equation*}
    \exp\left(-\frac{\|x-y\|_\hs{H}^2}{t}\right)
\end{equation*}
is roughly zero, unless $y$ is very close to $x$, where close is determined by $t$ (a small $t$ means that $y$ has to be really close to $x$).
The term
\begin{equation*}
    \int_{B_{\Delta(x)}(x)}\exp\left(-\frac{\|x-y\|_\hs{H}^2}{t}\right)\mathrm{d} \Normal(y\mid 0,Q)
\end{equation*}
measures how much weighted Gaussian mass is contained in $B_{\Delta(x)}(x)$, where the weighting is provided by $ \exp\left(-\frac{\|x-y\|_\hs{H}^2}{t}\right)$.
We therefore have a double localization. 
First, we only consider the $\Delta(x)$ ball around $x$, and then we indirectly localize even more using the exponential term, with smaller $t$ leading to stronger localization around $x$.
Note that two mechanisms make this term small. If $\|x\|_\hs{H}$ is very large, all the $y$ considered inside the integral will have a large norm, hence the Gaussian mass will be very small (due to the fast tail decay of the Gaussian measure).
Similarly, if $x$ is close to the decision boundary, so $\Delta(x)$ is very small, the ball over which the integral is defined, is very small, and hence has small weighted Gaussian mass.
Altogether, the weighting term
\begin{equation*}
    \phi_{t,Q}(x) = 1- 2\int_{B_{\Delta(x)}(x)}
            \exp\left(-\frac{\|x-y\|_\hs{H}^2}{t}\right)
            \mathrm{d} \Normal(y\mid 0,Q)
\end{equation*}
deems $x$ important if it has a large norm $\|x\|_\hs{H}$ (where large is relative to $Q$),
or if $x$ is close to the decision boundary.
Consider now the condition from the assumption above, which we can rewrite as 
\begin{equation*}
     \int_\hs{H} \phi_{t,Q}(x) |2\eta(x)-1|\mathrm{d}P_X(x) \leq C_Q t^{\alpha_Q}.
\end{equation*}
It states that the average predictability, weighted by $\phi_{t,Q}$, rapidly approaches 0 for $t \searrow 0$ (i.e., if we increase the localization in the weighting).

We now turn to \eqref{eq:assumption:distrslt:geometricNoiseAssumptionGaussian:2}.
To interpret this condition, define the weighting factor
\begin{equation*}
    \psi_{t,Q}(x)=   \int_\hs{H} \exp\left(-\frac1t \|x-y\|_\hs{H}^2\right) \mathrm{d} \Normal(y\mid 0,Q).
\end{equation*}
The $\exp\left(-\frac1t \|x-y\|_\hs{H}^2\right)$ is roughly zero for $y$ far away from $x$, where far away is relative to $t$.
In particular, for a small $t$, even a moderately close $y$ appears to be rather far away and makes the exponential term almost zero.
The weighting term
\begin{equation}
    \int_\hs{H} \exp\left(-\frac1t \|x-y\|_\hs{H}^2\right) \mathrm{d} \Normal(y\mid 0,Q).
\end{equation}
therefore measures how much Gaussian mass, reweighted by the exponential term above, is close to $x$.
Note that for very small $t$, this will be rather small (since the effective neighbourhood around $x$ that is considered in the integral is very small),
and it will also be small for very large $\|x\|$ (since then all $y$ close to $x$ also have large norm).

We can rewrite the condition from the assumption as
\begin{equation*}
     \int_\hs{H} \psi_{t,Q}(x)|2\eta(x)-1|\mathrm{d}P_X(x)
    \leq C_Q t^{\alpha_Q},
\end{equation*}
which means that the average predictability, reweighted by $\psi_{t,Q}(x)$, rapidly goes to zero for $t \searrow 0$.
Note that this is a very mild assumption, since as explained above, in general $\psi_{t,Q}(x)$ goes to zero for $t \searrow 0$.

\end{document}